\documentclass{article}

\usepackage{PRIMEarxiv}

\usepackage[utf8]{inputenc} 
\usepackage[T1]{fontenc}    
\usepackage{hyperref}       
\usepackage{url}            
\usepackage{booktabs}       
\usepackage{amsfonts}       
\usepackage{nicefrac}       
\usepackage{microtype}      
\usepackage{lipsum}
\usepackage{fancyhdr}       
\usepackage{graphicx}       
\graphicspath{{media/}}     

\usepackage{natbib}
\usepackage{amssymb}
\usepackage{amsthm}
\usepackage{algorithm}
\usepackage{algpseudocode}
\usepackage{xcolor} 
\usepackage{hyperref}

\newtheorem{theorem}{Theorem}
\newtheorem{remark}{Remark}
\newtheorem{lemma}{Lemma}
\newtheorem{definition}{Definition}
\newtheorem{assumption}{Assumption}

\usepackage{amsmath}
\usepackage{cleveref}

\Crefname{corollary}{Corollary}{Corollary}
\Crefname{assumption}{Assumption}{Assumptions}
\newcommand{\doublebar}[1]{\Bar{\Bar{#1}}}
\newcommand\cnorm[1]{\left\lVert #1 \right\rVert}

\usepackage{authblk}

\title{Sample Complexity Characterization for Linear Contextual MDPs
}

\author[1]{Junze Deng}
\author[2]{Yuan Cheng}
\author[3]{Shaofeng Zou}
\author[4]{Yingbin Liang}
\affil[1]{Department of ECE, The Ohio State University, \texttt{deng.942@osu.edu}}
\affil[2]{National University of Singapore, \texttt{yuan.cheng@u.nus.edu}}
\affil[3]{Department of ECE, University at Buffalo, \texttt{szou3@buffalo.edu}}
\affil[4]{Department of ECE, The Ohio State University, \texttt{liang.889@osu.edu}}

\begin{document}
\maketitle
\begin{abstract}
  Contextual Markov decision processes (CMDPs) describe a class of reinforcement learning problems in which the transition kernels and reward functions can change over time with different MDPs indexed by a context variable. While CMDPs serve as an important framework to model many real-world applications with time-varying environments, they are largely unexplored from theoretical perspective. In this paper, we study CMDPs under two linear function approximation models: Model I with context-varying representations and common linear weights for all contexts; and Model II with common representations for all contexts and context-varying linear weights. For both models, we propose novel model-based algorithms and show that they enjoy guaranteed $\epsilon$-suboptimality gap with desired polynomial sample complexity. In particular, instantiating our result for the first model to the tabular CMDP improves the existing result by removing the reachability assumption. Our result for the second model is the first-known result for such a type of function approximation models. Comparison between our results for the two models further indicates that having context-varying features leads to much better sample efficiency than having common representations for all contexts under linear CMDPs.
\end{abstract}


\section{Introduction}
Reinforcement learning (RL) \citep{sutton2018reinforcement} aims to optimize the interaction process between agents and the environment, and has succeeded in many practical applications, e.g., games \citep{silver2016mastering}, robotics~\citep{levine2016end,gu2017deep}, and recommendation systems \citep{li2010contextual}. Typically, Markov Decision Processes (MDPs) are employed to model these interaction processes \citep{bertsekas2011dynamic}.  In a single MDP, the transition kernel and reward function remain invariant across different episodes. The objective of the agent is to learn a policy that maximizes the cumulative reward under the {\em same} MDP throughout the interaction. 

However, many real-world applications involve {\em time-varying} transition kernels and reward functions. These scenarios are influenced by an additional variable known as a {\em context}. To capture such complexities, Contextual Markov Decision Processes (CMDPs) are utilized \citep{modi2018smooth_CMDP,levy2022Function_Approximation_CMDP,hallak2015contextual_first}. For instance, consider a multi-user recommendation system where transition probabilities and reward functions may differ significantly among users. Describing the diverse user behaviors using a single MDP becomes challenging. In contrast, a CMDP enables the modeling of a multi-user recommendation system by introducing a context-dependent transition kernel and a context-dependent reward function \citep{kabra2021potent}.

Although CMDPs have the capability to model various real-world applications with time-varying environments, study of their theoretical performance remains limited. Recently, \citet{sodhani2022block,modi2018smooth_CMDP} have investigated a particular class of CMDPs, known as Lipschitz CMDPs (or smooth CMDPs). 
Additionally, \citet{dong2020low_bell_rank,jiang2017low_bell_rank} studied the class of CMDPs with low Bellman rank.
Moreover,  \citet{levy2022tabular_contextual} explored the tabular CMDPs with a minimal reachability assumption and a finite state space.
To further generalize such study to a large or even infinite-state space, \citet{amani2022provably_lifilong} studied CMDPs with linear function approximation, where the reward function is context-dependent but the transition kernel is context-independent. 

In this paper, we further explore more general CMDP models with linear function approximation, 
where both the transition kernel and the reward function are context-dependent and are linear functions of features and weights, and the state space is large or even continuous. In particular, we study the following two models for CMDPs, both of which are well justified in practical applications (see \Cref{sec:linearcmdp} for motivating examples). 
\begin{itemize}

\item In Model I, the transition and the reward can be decomposed into a linear function of known context-varying representation and unknown common linear weights that are shared across all contexts. {\em This model generalizes the CMDP in \citet{amani2022provably_lifilong} with a fixed transition to context-varying transitions.}

\item In Model II, the transition kernel and the reward function can be decomposed into a linear function of known common representation and unknown context-varying linear weights. {\em Such a CMDP model has not been studied in the past.} 
\end{itemize}

We summarize our main contributions below. 




\textbf{Novel model-based algorithm.} 
For both models, we design model-based algorithms to enable the use of all historical data (generated under {\em different} environments) simultaneously for learning the transition model and the reward function. This is in contrast to the model-free algorithms designed for single linear MDPs \citep{jin2020linear_mdp} and linear CMDP with fixed transition kernel \citep{amani2022provably_lifilong}, where all all historical data (in the past episodes) are collected under a fixed transition kernel and can hence be used to directly estimate the value function.

\textbf{Novel bonus term design.} For Model II, the bonus term for promoting optimism is designed to be the {\em squared} norm of features, which is based on a novel decomposition of the value function uncertainty into the context-dependent and context-independent components, so that context-independent components can be bounded by the squared norm of features. This is in contrast to
the standard UCB bonus design (that adopts the norm itself) in the previous studies of function approximation for {\em fixed} transition kernels ~\citep{amani2022provably_lifilong,jin2020linear_mdp,hu2022nearly}. 

\textbf{Theoretical guarantee.} For both models, we provide provable upper bounds on the average sub-optimality gap and the corresponding sample complexity to achieve such a near-optimal performance. Specifically, for Model I, Our result improves that for tabular CMDP in \citep{levy2022tabular_contextual} by removing their reachability assumption, and thus enjoys a better sample efficiency if the reachability lower bound is small, which is often the case in practice \citep{agarwal2020small_pmin}. Our result for context-varying transitions has the same sample complexity as that in \citep{amani2022provably_lifilong} for CMDP with fixed transitions, indicating that our handling of context-varying transition does not incur additional sample complexity. For Model II, our result is the first in the literature established for such a type of function approximation models.

\section{Related Work}
\textbf{Contextual MDPs (CMDPs).} 
The study of CMDPs was originated by the work of \citep{hallak2015contextual_first}. Since then, several specialized classes of CMDPs have been explored, incorporating additional structural assumptions. One notable class is that of smooth CMDPs, which was proposed in \citep{modi2018smooth_CMDP}. 
The authors developed a framework for designing Probably Approximately Correct (PAC) algorithms specifically tailored for smooth CMDPs with a finite context space.  This work focused on achieving efficient and accurate decision-making in CMDPs by leveraging the smoothness assumptions. \citet{sodhani2022block} further studied the properties of Lipschitz block CMDPs. Another important class of CMDPs with low Bellman rank was introduced  in \citep{jiang2017low_bell_rank}. Further, function approximation techniques were used in \citep{levy2022Function_Approximation_CMDP} to obtain the sample complexity, where they assumed the access to an empirical risk minimization oracle. Then CMDPs with linear function approximation were studied in \citep{amani2022provably_lifilong}, where only the reward is context-dependent. More recently, tabular CMDPs with both context-dependent transition and context-dependent reward were studied in \citep{levy2022tabular_contextual}. In this paper, we focus on more general CMDPs with both context-dependent transition and context-dependent reward, where the transition kernel can be modeled with linear function approximation. Thus, our models include the CMDPs studied in \citep{levy2022tabular_contextual,amani2022provably_lifilong} as special cases.


\textbf{Contextual Bandits.} 
Contextual bandits can be viewed as a natural extension of the classical multi-armed bandit problem \citep{slivkins2019introduction}. In contextual bandit settings, additional information, known as context, is provided to the decision-making agent. This context influences the reward associated with each action. Further, contextual bandits
can be viewed as special cases of CMDPs with horizon one. The major challenge here lies in estimating the reward function, which is commonly solved using regression-based methods, e.g., \citep{chu2011contextual_bandit_linear_reward,foster2018regression_contextual_bandits,agarwal2014contextual_bandits,xu2020contextual_bandits}. 


\textbf{Adversarial RL and Nonstationary RL.} There have been two other lines of research that model time-varying transition kernels and reward functions. The first line is on adversarial RL~\citep{DBLP:journals/jmlr/NeuGS12,DBLP:conf/nips/ZiminN13,DBLP:conf/colt/LykourisSSS21,DBLP:conf/icml/RosenbergM19,jin2019learning,xiong2021non}, which allows time variations in reward functions but assumes an identical transition kernel over time. The second line is on nonstationary RL~\citet{DBLP:journals/corr/abs-2110-08984,DBLP:conf/icml/MaoZZSB21,DBLP:journals/corr/abs-2010-12870,DBLP:journals/corr/abs-2010-04244,cheng2023provably,xiong2020finite,feng2023non}, where both transition kernels and reward functions can be time-varying. Note that both adversarial RL and nonstationary RL assume that rewards and/or samples taken under current transitions can be used only in the {\em next} episode, whereas the agent in contextual MDPs can access and exploit the rewards and samples taken under the {\em current} MDP (i.e., the context) to achieve better performance. 

\textbf{Linear MDPs.}
Linear function approximation in Markov Decision Processes (MDPs) has been widely explored in the literature, e.g.,\citet{jin2020linear_mdp,du2019linear_mdp,wang2019linear_mdp,he2021logarithmic_linear_mdp,zhang2021improved_linear_mdp,chu2011contextual_bandit_linear_reward}. The use of linear function approximation allows for efficient and scalable representation of value functions or policies in MDPs, facilitating the handling of high-dimensional state spaces. For liner MDPs, \citep{jin2020linear_mdp} developed  a standard framework of  Upper Confidence Bound (UCB) based model-free algorithms, and \citep{he2022nearly,hu2022nearly} designed algorithms that are nearly minimax optimal. 
Most previous works on linear MDPs \citep{jin2020linear_mdp,amani2022provably_lifilong} adopt model-free approaches to approximate the value function directly. However, these approaches are not directly applicable to settings where the transition kernel is changing over time, i.e., context-varying. In this paper, we develop model-based approaches that effectively take advantage of historical data collected under time-varying transition kernels for better model estimation.

\section{Preliminaries and Problem Formulation}

\textbf{Notation.} For a positive integer $H$, let  $[H]:= \{1,2,...,H\}$. For a vector $x$, define the vector norm of $x$ w.r.t. a positive symmetric matrix $A$ by $\cnorm{x}_A:=\sqrt{xA^\top x}$. For a finite set $\mathcal{A}$, we use $\mathcal{U(A)}$ to denote the uniform distribution over the set $\mathcal{A}$.

\subsection{Contextual MDPs}
A Contextual Markov Desicion Process (CMDP) can be described by a tuple $(\mathcal{W},\mathcal{S},\mathcal{A},\mathcal{M})$, where $\mathcal{W}$ is the context space, which can be continuous or infinite, $\mathcal{S}$ is the state space, which can be continuous and infinite, $\mathcal{A}$ is the finite action space with the cardinality $K$, and the mapping $\mathcal{M}$ maps a context $w\in\mathcal{W}$ to a Markov Decision Process (MDP): $\mathcal{M}(w) = (\mathcal{S},\mathcal{A},P_w,r_w,H)$. Specifically, $H$ is the horizon length, and $P_{w}=\{P_{w,h}\}_{h=1}^H$ denotes the time-dependent and context-dependent transition kernel, i.e., $P_{w,h}(s'|s,a)$ is the probability of reaching state $s'$ in the next step given the state-action pair $(s,a)$ at step $h$ when the context is $w$. 
For convenience, we denote $P_h(s'|s,a,w) := P_{w,h}(s'|s,a)$.  Furthermore, $r_{w}=\{r_{w,h}\}_{h=1}^H$ denotes the deterministic reward function where $r_{w,h}:\mathcal{S}\times\mathcal{A}\rightarrow [0,1]$ is the reward function at step $h$ given the context is $w$. We also write $r_h(s,a,w) := r_{w,h}(s,a)$.

At the beginning of each episode, a context $w$ is drawn randomly from a distribution $q$, and the agent then experiences the MDP $\mathcal{M}(w)$ for the current episode. Each episode starts with a fixed initial state $s_1$ independent of the context. At each step $h$, the agent observes a state $s_h$, takes an action $a_h\in \mathcal{A}$ under a possibly context-dependent policy $\pi_w$, receives a reward $r_h(s_h,a_h,w)$, and the system transits to the next state $s_{h+1}$ following the probability $P_h(s_{h+1}|s_h,a_h,w)$.

For a given MDP, we use $a_h\sim\pi$ to denote that an action $a_h$ is selected according to a policy $\pi$. We use $s_h\sim (P,\pi)$ to denote the distribution of $s_h$ applying policy $\pi$ under the transition kernel $P$ for $h-1$ steps. Then we use $\mathbb{E}_{(s_h,a_h)\sim(P,\pi)}$ to denote the expectation over states $s_h\sim (P,\pi)$ and actions $a_h\sim\pi$. Given an MDP: $ M = (\mathcal{S},\mathcal{A},P,r,H)$ and a policy $\pi$, we denote by $V^{\pi}_{h,M}$  the value function under the MDP $M$ and the policy $\pi$ starting from step $h$ and state $s_h$, i.e., 
\[\textstyle V^{\pi}_{h, M} = \mathbb{E}_{(s_{h'},a_{h'})\sim(P,\pi)}\left[ \sum_{h'=h}^H r_{h'}(s_{h'},a_{h'}) | s_h \right].\] 
For simplicity, we use $V^{\pi}_{M}$ to denote $V^{\pi}_{1,M}$. 
For the MDP $\mathcal{M}(w) = (\mathcal{S},\mathcal{A},P_w,r_w,H)$, we use $\pi_w^*$ to denote an optimal policy that maximizes the value function w.r.t.\ the corresponding context $w$, i.e., 
\[\textstyle \pi_w^*=\arg\max_{\pi}V^{\pi}_{\mathcal M(w)}.\] 

For the CMDP problem, our goal is to  obtain a series of policies $\{ \pi^n_w\}_{n=1}^N$ over time steps $n=1,\ldots, N$ that minimize the following \textit{average sub-optimality gap}, which takes the expectation over context and the average over time:
\begin{equation*} 
\textstyle \frac{1}{N}\sum_{n=1}^N\underset{w\sim q}{\mathbb{E}} \left [V^{\pi^*_{w}} _{\mathcal{M}(w)} - V^{\pi^n_{w}} _{\mathcal{M}(w)}\right].
\end{equation*}
A sequence of policy $\{\pi_w^n\}_{n=1}^N$ is $\epsilon$-optimal if the average suboptimality gap is less than $\epsilon$.

\subsection{Two Linear Function Approximation Models for CMDPs}\label{sec:linearcmdp}
In this paper, we study two different linear function approximation models with context-varying transition kernels and reward functions. 

The first model defined below has context-varying representations and the same linear weights for all contexts.
\begin{definition}[Model I: CMDPs with varying representation]\label{def:model 1}
 The environment transition kernel $P_w$ admits a linear decomposition of a known representation $\phi_h:\mathcal{S}\times\mathcal{A}\times\mathcal{W} \rightarrow \mathbb{R}^d$ and an unknown linear weights function $\mu_h:\mathcal{S} \rightarrow \mathbb{R}^d$ as follows:
\begin{equation}\label{eq:model 1 trans}
    P_h(s'|s,a,w) = \langle \phi_h(s,a,w), \mu_h(s') \rangle.
\end{equation}
Moreover, the reward function $r_w$ admits a similar linear decomposition of a known feature function $\psi_h:\mathcal{S}\times\mathcal{A}\times\mathcal{W} \rightarrow \mathbb{R}^d$ and an unknown linear weights $\eta_h\in \mathbb{R}^d$:
\begin{equation}\label{eq:model 1 reward}
    r_h(s,a,w) = \langle \psi_h(s,a,w), \eta_h \rangle.
\end{equation}
For normalization, we assume $\cnorm{\phi_h(s,a,w)}_2\le 1$ and $\cnorm{\psi_h(s,a,w)}_2\le 1$ for any $(s,a,w)\in\mathcal{S}\times\mathcal{A}\times\mathcal{W}$. For any function $g\rightarrow[0,1]$, we assume $\cnorm{\int \mu_h(s)g(s)ds}_2\le\sqrt{d}$. Furthermore, we assume that $\cnorm{\eta_h}_2\le\sqrt{d}$.
\end{definition}

Model I captures read-world applications in which the environment transition depends on a few dominating features, where the features can vary rather quickly but the weights are relatively steady. For example, in multi-user recommendation systems, features are user-dependent and change across uses. 
But the importance of individual features can remain the same regardless of user identity, meaning that the roles that these features play in the system dynamics are the same for all users.


In this model, we take the following assumption widely used in the RL literature \citep{cheng2023ass:learning_model,levy2022tabular_contextual,sun2019ass:learning_model}.
\begin{assumption}\label{ass:mu(s)}
    The learning agent has access to a finite model class $ \Psi_1$, where the true model $\mu_h(s)\in\Psi_1$ for any $h\in[H]$.
\end{assumption}
In this paper, we assume that the cardinality of the function class is finite. It can be extended to an infinite function class with bounded statistical complexity such as a bounded covering number \citep{sun2019ass:learning_model,uehara2021representation}.

The second model has the same features for all contexts but context-varying linear weights. 
\begin{definition}[Model II: CMDPs with varying linear weights]\label{def:model 2}
 The environment transition kernel $P_w$ admits a linear decomposition with a known representation $\phi_h:\mathcal{S}\times\mathcal{A} \rightarrow \mathbb{R}^d$ and an unknown linear weight function $\mu_h:\mathcal{S}\times\mathcal{W} \rightarrow \mathbb{R}^d$:
\begin{equation}\label{eq:model 2 trans}
    P_h(s'|s,a,w) = \langle \phi_h(s,a), \mu_h(s',w) \rangle.
\end{equation}
Moreover, the reward function $r_w$ admits a similar linear decomposition of a known feature function $\psi_h:\mathcal{S}\times\mathcal{A} \rightarrow \mathbb{R}^d$ and an unknown linear weights function: $\eta_h:\mathcal{W} \rightarrow \mathbb{R}^d$:
\begin{equation}\label{eq:model 2 reward}
    r_h(s,a,w) = \langle \psi_h(s,a), \eta_h(w) \rangle.
\end{equation}
For normalization, we assume $\cnorm{\phi_h(s,a)}_2\le 1$ and $\cnorm{\psi_h(s,a)}_2\le 1$ for any $(s,a)\in\mathcal{S}\times\mathcal{A}$. For any $w\in\mathcal{W}$ and function $g\rightarrow[0,1]$, we assume $\cnorm{\int \mu_h(s,w)g(s)ds}_2\le\sqrt{d}$. Furthermore, we assume that $\cnorm{\eta_h(w)}_2\le\sqrt{d}$ for any $w\in\mathcal{W}$.
\end{definition}
Model II captures many read-world applications where environmental transitions rely on relatively stable features but with varying weights depending on contexts. For example, suppose an autonomous car drives over several traffic patterns. These traffic patterns can share common features, 
but the role that these features play in the system dynamics can change with traffic patterns, i.e., the linear weights are different across contexts.


For the second model, we take the following assumption similar to the previous one.
\begin{assumption}\label{ass:mu(sw)}
    The learning agent has access to two finite model classes $\Psi_2$ and $\Psi_3$, where the true model $\mu_h(s,a)\in\Psi_2$ and $\eta_h(s,w)\in\Psi_3$ for any $h\in[H]$.
\end{assumption}


The above two linear contextual MDP models take the linear function approximation structure in the standard linear MDP \citep{jin2020linear_mdp} for each context. However, solving a CMDP is significantly harder than solving a single MDP due to the context-varying transition kernel. The key challenges lie in how to use historical data taken over different MDPs to benefit the learning in future MDPs; and the context and space can be infinite.


\section{Model I: Varying Representation}
In this section, we consider the 
CMDPs with varying representations as defined in \Cref{def:model 1}.

\textbf{Technical Challenge:} Although there exists a line of works in single linear MDPs \citep{jin2020linear_mdp} and linear CMDP with fixed transition kernel and context-dependent reward \citep{amani2022provably_lifilong}, directly extending their model-free algorithms to our model is challenging. This is because in those studies, the transition kernel is fixed, and hence  
all historical data (in the past episodes) can be used to directly estimate the value function, because they are generated via the same environment. This is not feasible for our model, because the transition kernel is context-varying. Therefore, 
we will design a model-based algorithm, which makes it convenient to use all historical data (generated under different environments) to learn the transition model and the reward function first, and then conduct planning for each context. 


\subsection{Algorithm}
We propose a novel model-based algorithm, particularly designed to exploit data generated by context-varying environments. As we mention above, the previous model-free algorithms for fixed transition kernel \citep{jin2020linear_mdp,amani2022provably_lifilong} are not applicable here. Our algorithm is presented in \Cref{alg1}. 



\textbf{Estimation of transition kernel and reward.} At the beginning of each episode $n$, the agent observes a context $w_n$. 
For each step $h\in[H]$, the agent uses historical  data $\{(s_h^\tau,a_h^\tau,w_\tau,r_h^\tau)\}_{\tau=1}^{n-1}$ to estimate the reward function as follows: 
\begin{equation*}
    \hat{\eta}_h^n = \underset{\eta_h \in \mathbb{R}^d}{\mathrm{argmin}} \sum_{\tau=1}^{n-1}\left(\langle\eta_h, \psi_h(s_h^\tau,a_h^\tau,w_\tau)\rangle - r_h^\tau \right)^2 + \xi_n \cnorm{\eta_h}_2^2, 
\end{equation*}
where $\xi_n>0$ is some constant to be specified later.
Moreover, for each step $h\in[H]$, the agent uses an maximum likelihood estimation (MLE) oracle on the collected data set $\mathcal{D}_h^n =\{(s_h^\tau,a_h^\tau,s_{h+1}^\tau,w_\tau)\}_{\tau=1}^{n-1} $ to estimate the weights of transition kernel  as follows:
\begin{align}
\label{eq:mle}
   \hat{\mu}_h^n &= \mathrm{MLE}(\mathcal{D}_h^n) \nonumber\\ 
   &= \underset{\mu_h \in \Psi_1}{\mathrm{argmax}} \sum_{(s,a,s',w)\in \mathcal{D}_h^n} \mathrm{log} \langle\phi_h(s,a,w),\mu_h(s')\rangle.
\end{align}

\textbf{Design of UCB bonus terms.} We define the following two matrices:
\begin{align}
    \hat{\Sigma}_h^n &= \sum_{\tau=1}^{n-1} \phi_h(s_h^\tau,a_h^\tau,w_\tau)\phi_h(s_h^\tau,a_h^\tau,w_\tau)^\top + \lambda_nI,\nonumber\\
    \hat{\Lambda}_h^n &= \sum_{\tau=1}^{n-1} \psi_h(s_h^\tau,a_h^\tau,w_\tau)\psi_h(s_h^\tau,a_h^\tau,w_\tau)^\top + \xi_nI, \label{eq:hat Lambda}
\end{align}
where $\lambda_n = \gamma_1d\mathrm{log}(2nH/\delta)$, $\xi_n = \gamma_2d\mathrm{log}(2nH/\delta)$, $\gamma_1, \gamma_2 = \mathcal{O}(1)$ and $I$ denotes the identity matrix.
The UCB bonus terms for the transition kernel and the reward are then defined as follows:
\begin{align}
    \hat{b}_h^n (s,a,w) &= \min\{\alpha_n \cnorm{\phi_h(s,a,w)}_{(\hat{\Sigma}_h^n)^{-1}},H\},\label{eq:hat b}\\
    \hat{c}_h^n (s,a,w) &= \min\{\beta_n \cnorm{\psi_h(s,a,w)}_{(\hat{\Lambda}_h^n)^{-1}},1\}, \label{eq:hat c}
\end{align}
where $\alpha_n = 5H\sqrt{2\lambda_n d + 4\mathrm{log}(2nH|\Psi_1|/\delta)}$  and $\beta_n = \sqrt{d\xi_n}$. Note that $\lambda_n,\xi_n = \mathcal{O}(d)$, and therefore, $\alpha_n = \widetilde{\mathcal{O}}(dH)$ and $\beta_n = \mathcal{O}(d)$. Moreover, the parameter $\alpha_n$ depends on the size of the model class $\Psi_1$. The design of $\alpha_n$ is to bound the gap of value functions due to the estimation error of the transition kernel. The design of $\beta_n$ is to bound the gap of value functions due to the estimation error of the reward function. See further details in \Cref{remark:point wise bound}.

Then the agent uses the estimated function $\hat{\mu}_h^n(s)$ to update the transition as \[\hat{P}_h^n (s'|s,a,w) = \langle \phi_h(s,a,w), \hat{\mu}_h^n(s') \rangle,\] and uses the estimated reward weights $\hat{\eta}_h^n$ with defined bonus terms to update the optimistic reward function as 
\begin{equation}\label{eq:hat r}
    \hat{r}_h^n(s,a,w) = \hat{f}_h^n(s,a,w) + \hat{b}_h^n(s,a,w) + \hat{c}_h^n(s,a,w),
\end{equation} where $\hat{b}_h^n$ and $\hat{c}_h^n$ are defined in \cref{eq:hat b,eq:hat c} and
\begin{equation}
    \hat{f}_h^n(s,a,w) = \left \{
    \begin{array}{ll}
        \langle\hat{\eta}_h^n, \psi_h(s,a,w)\rangle  & \mathrm{if}  \langle\hat{\eta}_h^n, \psi_h(s,a,w)\rangle\in [0,1]\\
        1 & \mathrm{if}  \langle\hat{\eta}_h^n, \psi_h(s,a,w)\rangle>1\\
        0 & \mathrm{if}  \langle\hat{\eta}_h^n, \psi_h(s,a,w)\rangle<0
    \end{array}. \right.
\end{equation}\label{def:hat f}

\begin{remark}\label{remark:point wise bound}
    It can be shown that with high probability, we have: \[|(P_{w,h}-\hat{P}^n_{w,h}) V_{h+1,\mathcal{M}^{(r,\hat{P}^n)}(w)}(s,a)|\le \hat{b}_h^n(s,a,w)\] and \[|\langle\eta_h-\hat{\eta}^n_h,\psi_h(s,a,w)  \rangle|\le \hat{c}_h^n(s,a,w)\] for any $(s,a,w)\in \mathcal{S}\times \mathcal{A} \times \mathcal{W}$.
\end{remark}


\textbf{Planning and exploration.} For the estimated MDP $\hat{\mathcal{M}}(w) = (\mathcal{S},\mathcal{A},H,\hat{P}^n_w,\hat{r}^n_w)$, the agent defines a truncated value function iteratively using the optimistic reward function and the estimated transition kernel: 
\begin{align}\label{def:doublebar}
    &\doublebar{Q}^{\pi_w}_{h,\hat{\mathcal{M}}(w)}(s_h,a_h) = \min\left\{ 3H,\hat{r}^n_h(s,a,w)+\hat{P}^n_{h,w}\doublebar{V}^{\pi_w}_{h+1,\hat{\mathcal{M}}(w)}(s_h,a_h)\right\}, \nonumber\\
    &\doublebar{V}^{\pi_w}_{h,\hat{\mathcal{M}}(w)}(s_h) = \underset{\pi_w}{\mathbb{E}}\left[ \doublebar{Q}^{\pi_w}_{h,\hat{\mathcal{M}}(w)}(s_h,a_h) \right].
\end{align}
Note that the truncation threshold of $3H$ is specially designed to provide a valid bound for the optimistic reward, which consists of three elements including the bonus term for the kernel estimation error, the bonus term for the reward estimation error and the estimated reward, and each of their corresponding value functions can be truncated to $H$.


In the first episode, the agent starts with a random policy and executes such a policy to collect data. Then in the following episodes, the agent finds an optimal context-varying policy $\pi^n_w$ of the truncated value function:
\begin{align}\label{eq:planningpolicy}
    \pi^n_w=\arg\max_\pi \doublebar{V}_{\hat{\mathcal{M}}(w)}^\pi.
\end{align}
Then the agent executes the policy $\pi^n_w$ in the current episode and collects data. 

\begin{algorithm}
\caption{CMDPs with varying representation}\label{alg1}
Initialization: regularizers $\lambda_n, \xi_n $, model class $\Psi_1$, MLE data set $\mathcal{D}_h^n = \emptyset$
\begin{algorithmic}[1]
\For{$n = 1,...,N$}
    \State observe context $w_n$.
    \If{$n = 1$}
        \State set $\pi^1_{w_1}$ as a random policy
    \Else
        \For{$h = 1,...,H$}
            \State $\hat{\eta}_h^n = \left(\hat{\Lambda}^n_h \right)^{-1} \cdot \sum_{\tau = 1}^{n-1} \psi_h(s_h^\tau,a_h^\tau,w_\tau) r_h^\tau$, where ${\hat{\Lambda}}^n_h$ is defined in \cref{eq:hat Lambda}.
            \State $\hat{\mu}_h^n = \mathrm{MLE}(\mathcal{D}_h^n)$, where $\mathrm{MLE}(\mathcal{D}_h^n)$ is defined in \cref{eq:mle}
            \State $\hat{P}_h^n (s'|s,a,w) = \langle \phi_h(s,a,w), \hat{\mu}_h^n(s') \rangle$.
            \State  $\hat{r}_h^n(s,a,w) = \hat{f}_h^n(s,a,w) + \hat{b}_h^n(s,a,w) + \hat{c}_h^n(s,a,w)$, where  $\hat{r}_h^n$ is defined in \cref{eq:hat r}.
        \EndFor
        \State $\hat{\mathcal{M}}(w) = (\mathcal{S},\mathcal{A},H,\hat{P}^n_w,\hat{r}^n_w)$.
        \State $\pi^n_{w_n} = \mathrm{argmax}_\pi\ {\doublebar{V}}^{\pi}_{\hat{\mathcal{M}}(w_n)}$, where ${\doublebar{V}}^{\pi}_{\hat{\mathcal{M}}(w_n)}$ is defined in \cref{def:doublebar}
    \EndIf
        \State Execute policy $\pi^n_{w_n}$, collect the trajectory: $s^n_1, a^n_1, r^n_1,...,s^n_H, a^n_H, r^n_H$.
        \State $\mathcal{D}_h^{n+1}=\mathcal{D}_h^n\cup \{(s^n_h, a^n_h, s^n_{h+1}, w_n)\} $ for $h\in[H]$.
\EndFor
\end{algorithmic}
\end{algorithm}
\subsection{Theoretical Analysis}

In this section, we develop an upper bound on the average sub-optimality gap and characterize the required sample complexity for finding a near-optimal policy in \Cref{thm:model 1}. The detailed proof of \Cref{thm:model 1} is deferred to Appendix~A.

\begin{theorem} \label{thm:model 1} 
Consider a CMDP with varying representations as defined in \Cref{def:model 1}. 
Under \Cref{ass:mu(s)}, for any $\delta\in(0,1)$, with probability at least $1-3\delta/2$, the sequence of policies $\{\pi_{w_n}^n\}_{n=1}^N$ generated by \Cref{alg1}  satisfies that
\begin{equation*}
\frac{1}{N}\sum_{n=1}^N  \underset{w\sim q}{\mathbb{E}} \left [V^{\pi^*_{w}} _{\mathcal{M}(w)} - V^{\pi^n_{w}} _{\mathcal{M}(w)}\right] \le \left(42H\sqrt{2d\lambda_N+4\mathrm{log}(2HN|\Psi_1|/\delta)}+6\sqrt{d\xi_N}\right) \cdot H \sqrt{\frac{2d}{N}}  \sqrt{\mathrm{log}\left(1+\frac{N}{d\lambda} \right)},
\end{equation*}
where $\lambda = \min\{ \lambda_1,\xi_1 \}$. To achieve an $\epsilon$ average sub-optimality gap, at most $\mathcal{O}\left(\frac{H^4d^3\log(|\Psi_1|/\delta)}{\epsilon^2}\right) $ episodes are needed.
\end{theorem}
We highlight the significance of \Cref{thm:model 1} via its comparisons with highly relevant existing studies as follows. \textbf{(i)} Our result generalizes that in \citep{amani2022provably_lifilong} with only context-dependent reward to both context-dependent transition kernel and context-dependent reward. Our result has the same sample complexity as that in \citep{amani2022provably_lifilong} indicating that our handling of context-varying transition does not incur additional sample complexity compared to the setting with a fixed transition kernel. 
\textbf{(ii)} Our result improves that in \citep{levy2022tabular_contextual} by removing the reachability assumption required for tabular MDP, and thus enjoys a better sample efficiency if  $q_{\mathrm{min}}$ is small, which is often the case in practice \citep{agarwal2020small_pmin}. 
\textbf{(iii)} Our sample complexity for CMDPs is $\tilde{\mathcal{O}}(\frac{H^4d^3}{\epsilon^2})$, which is the same as that of LSVI-UCB for a single MDP \citep{jin2020linear_mdp}. Thus, with the same sample efficiency, our approach can also solve CMDPs,  as long as the contexts share the common linear weights.


\section{Model II: Varying Linear Weights}
In this section, we consider Model II with unknown context-varying linear weights and the same known features for all contexts as defined in \Cref{def:model 2}. 

\textbf{Technical Challenge:} The bonus term design to promote optimism for Model II cannot apply the standard UCB bonus term in the previous studies of function approximation for fixed transition kernels ~\citep{amani2022provably_lifilong,jin2020linear_mdp,hu2022nearly}. The reason is that the features (i.e., representations) in Model II are fixed, and hence cannot fully upper-bound the context-varying value function gap due to the estimation error of the transition kernel and the reward function. Therefore, we will develop a novel decomposition of the value function gap into the context-dependent and context-independent components, so that context-independent components can be bounded by the squared norm of features.


\subsection{Algorithm}


 We adopt a model-based design to learn the transition model first, and conduct planning for each new context. The detailed algorithm is presented in \Cref{alg2}. 

\textbf{Estimation of transition kernel and reward.} At the beginning of episode $n$, the agent observes a context $w_n$. The agent uses collected data $\mathcal{D}_h^n=\{(s_h^{\tau},a_h^\tau,s_{h+1}^\tau,r_h^\tau,w_\tau)\}_{\tau=1}^{n-1}$ to estimate the transition by an MLE oracle:
\begin{align*}
    \widetilde{\mu}_h^n &= \mathrm{MLE}(\mathcal{D}_h^n) \\
    &= \underset{\mu_h \in \Psi_2}{\mathrm{argmax}} \sum_{(s,a,s',w)\in \mathcal{D}_h^n} \mathrm{log} \langle\phi_h(s,a),\mu_h(s',w)\rangle
\end{align*}
and the reward function by a least square regression (LSR) oracle: 
\begin{align*}
    \widetilde{\eta}_h^n &= \mathrm{LSR}(\mathcal{D}_h^n)\\
    &= \underset{\eta_h \in \Psi_3}{\mathrm{argmin}} \sum_{(s,a,r,w)\in \mathcal{D}_h^n} \left(\langle\psi_h(s,a),\eta_h(w)\rangle - r\right)^2.
\end{align*}
\textbf{Novel design of UCB bonus terms.} 
We first define the following matrices:
\begin{align}
    \widetilde{\Sigma}_h^n &= \sum_{\tau=1}^{n-1} \phi_h(s_h^\tau,a_h^\tau)\phi_h(s_h^\tau,a_h^\tau)^\top + \widetilde\lambda_nI,\nonumber\\
    \widetilde{\Lambda}_h^n &= \sum_{\tau=1}^{n-1} \psi_h(s_h^\tau,a_h^\tau)\psi_h(s_h^\tau,a_h^\tau)^\top + \widetilde\xi_nI, \label{widSigma}
\end{align}
where $\widetilde{\lambda}_n = \widetilde{\gamma}_1d\mathrm{log}(2nH/\delta)$, $\widetilde{\xi}_n = \widetilde{\gamma}_2d\mathrm{log}(2nH/\delta)$ and $\widetilde{\gamma}_1, \widetilde{\gamma}_2 = \mathcal{O}(1)$.

To design a UCB bonus term, we first note that the following important remark.
\begin{remark}\label{remark:expecation}
It can be shown (see \Cref{lemma:linear diff expect} in the Appendix) that the value function gap $|(P_{w,h}-\widetilde{P}^n_{w,h}) V_{h+1,\mathcal{M}^{(r,\widetilde{P}^n)}}(s,a)|$ for any $(s,a,w)\in \mathcal{S}\times \mathcal{A}\times \mathcal{W}$ can be decomposed (via an upper bound) into the context-independent and context-dependent components, where the features $\phi_h(s,a)$ can play a role only in the context-independent term because the features do not change over contexts in Model II. 
\end{remark}
\Cref{lemma:linear diff expect} in the Appendix also indicates that only the {\em squared} norm of the features $\phi_h(s,a)$ will serve as valid bonus terms to upper-bound the context-independent component of the value function gap. This is very different from the standard UCB bonus terms based on the norm of features in the previous studies for fixed transition kernels~\citep{amani2022provably_lifilong,jin2020linear_mdp,hu2022nearly}.
%
%
Thus, we design the UCB bonus term $\widetilde{b}^n_h(s,a)$ as follows:
\begin{equation}\label{eq:wid b}
\widetilde{b}_h^n (s,a) = \min \left\{ \widetilde{\alpha}_n \cnorm{\phi_h(s,a)}^2_{(\widetilde{\Sigma}_h^n)^{-1}},H\right\},
\end{equation}
where $\widetilde{\alpha}_n = \frac{25}{2\sqrt{K}}CH\sqrt{dN}$ with $C=\frac{p_{\mathrm{max}}}{p_{\mathrm{min}}}$, and $p_{\mathrm{max}}$ and $p_{\mathrm{min}}$ are defined in \Cref{ass:pmin}. 
Similarly, we define the bonus term for the estimation error of the reward function as:
\begin{equation}\label{eq:wid c}
\widetilde{c}_h^n (s,a) = \min \left\{\widetilde{\beta}_n \cnorm{\psi_h(s,a)}^2_{(\widetilde{\Lambda}_h^n)^{-1}},1\right\},
\end{equation}
where $\widetilde{\alpha}_n = \frac{25}{2\sqrt{K}}C\sqrt{dN}$. 
Then the agent uses the estimated function $\widetilde{\mu}_h^n(s,w)$ to update the estimate of the transition kernel as 
\begin{equation*}
    \widetilde{P}_h^n(s'|s,a,w) = \langle\phi_h(s,a),{\widetilde{\mu}_h^n(s',w)} \rangle,
\end{equation*} 
and further updates the optimistic reward function using the estimated function $\widetilde{\eta}_h^n(w)$ and the above defined bonus terms as 
\begin{equation} \label{eq:wid r}
    \widetilde{r}^n_h(s,a,w) = \widetilde{f}_h^n(s,a,w) +\widetilde{b}_h^n(s,a) +\widetilde{c}_h^n(s,a),
\end{equation}
where $\widetilde{b}_h^n$ and $\widetilde{c}_h^n$ are defined in \cref{eq:wid b,eq:wid c} and
\begin{equation}\label{def:wid f}
    \widetilde{f}_h^n(s,a,w) = \left \{
    \begin{array}{ll}
        \langle\widetilde{\eta}_h^n(w), \psi_h(s,a)\rangle  & \mathrm{if}  \langle\widetilde{\eta}_h^n(w), \psi_h(s,a)\rangle\in [0,1]\\
        1 & \mathrm{if}  \langle\widetilde{\eta}_h^n(w), \psi_h(s,a)\rangle>1\\
        0 & \mathrm{if}  \langle\widetilde{\eta}_h^n(w), \psi_h(s,a)\rangle<0
    \end{array}. \right.
\end{equation}

\textbf{Planning and exploration.} 
For the estimated MDP $\widetilde{\mathcal{M}}(w) = (\mathcal{S},\mathcal{A},H,\widetilde{P}^n_w,\widetilde{r}^n_w)$, the agent defines a truncated value function iteratively using the optimistic reward function and the estimated transition kernel: 
\begin{align}\label{def:doublebar wid}
    &\doublebar{Q}^{\pi_w}_{h,\widetilde{\mathcal{M}}(w)}(s_h,a_h) = \min\left\{ 3H,\widetilde{r}^n_h(s,a,w)+\widetilde{P}^n_{h,w}\doublebar{V}^{\pi_w}_{h+1,\widetilde{\mathcal{M}}(w)}(s_h,a_h)\right\}, \nonumber\\
    &\doublebar{V}^{\pi_w}_{h,\widetilde{\mathcal{M}}(w)}(s_h) = \underset{\pi_w}{\mathbb{E}}\left[ \doublebar{Q}^{\pi_w}_{h,\widetilde{\mathcal{M}}(w)}(s_h,a_h) \right].
\end{align}
In the first episode, the agent starts with a random policy and executes such a policy to collect data. Then in the following episodes, with the estimated transition kernel $\widetilde{P}^n_w$ and the optimistic reward $\widetilde{r}^n_w$, the agent finds the optimal context-dependent policy $\pi^n_w$ for the MDP $\widetilde{\mathcal{M}}(w) = (\mathcal{S},\mathcal{A},H,\widetilde{P}^n_w,\widetilde{r}^n_w)$ as \[\pi^n_{w_n} = \mathrm{argmax}\ \doublebar{V}^{\pi}_{\widetilde{\mathcal{M}}(w_n)}.\]
To collect data, 
the agent does not simply execute the policy $\pi_{w_n}^n$ in the entire episode. Instead, for each $h\in[H]$, the agent first executes $\pi_{w_n}^n$ for $h$ steps, and then chooses the next action according to a uniform distribution $\mathcal U(\mathcal{A})$ and observes the next state. In this way, $s_h^n$ follows the distribution of $(P,\pi_{w_n}^n)$, and $a_h^n$ follows the uniform distribution. 
Such a uniform choice of actions provides a context-independent distribution of the action, which helps the agent to use history data from all previous contexts to estimate the transition and the reward function of the current context. 

\begin{algorithm}
\caption{CMDPs with varying linear weights}\label{alg2}
Initialization: regularizers $\widetilde{\lambda}_n, \widetilde{\xi}_n$, model classes $\Psi_2$ and $\Psi_3$, MLE data set $\mathcal{D}_h^n = \emptyset$.
\begin{algorithmic}[1]
\For{$n = 1,...,N$}
    \State observe context $w_n$.
    \If{$n = 1$}
        \State set $\pi^1_{w_1}$ as a random policy.
    \ElsIf{$n\ge 2$}
        \For{$h = 1,...,H$}
            \State $\widetilde{\mu}_h^n = \mathrm{MLE}(\mathcal{D}_h^n)$, $\widetilde{\eta}_h^n = \mathrm{LSR}(\mathcal{D}_h^n)$.
            \State  $\widetilde{P}_h^n (s'|s,a,w) = \langle \phi_h(s,a), \widetilde{\mu}_h^n(s',w) \rangle$.
        \EndFor
        \State $\widetilde{r}^n_h(s,a,w) = \widetilde{f}_h^n(s,a,w) +\widetilde{b}_h^n(s,a) +\widetilde{c}_h^n(s,a)$, where $\widetilde r _h^n$ is defined in \cref{eq:wid r}
        \State $\widetilde{\mathcal{M}}(w) = (S,A,H,\widetilde{P}^n_w,\widetilde{r}^n_w)$.
        \State  $\pi^n_{w_n} = \mathrm{argmax}\ \doublebar{V}^{\pi}_{\widetilde{\mathcal{M}}(w_n)}$.
    \EndIf
    \For{$h = 1,...,H$}
        \State use $\pi^{n}_{w_n}$: roll into $s_h$, take an action uniformly $a_h\sim \mathcal{U}(A)$, reach next state $s_{h+1}$.
        \State collect trajectory $s_1^n,a_1^n,...,s_h^n,a_h^n,s_{h+1}^n$.
        \State $\mathcal{D}_h^{n+1}=\mathcal{D}_h^n\cup \{(s_h^n,a_h^n,s_{h+1}^n,r^n_h, w_n)\}$.
    \EndFor
\EndFor
\end{algorithmic}
\end{algorithm}

\subsection{Theoretical Analysis}
In this subsection, we develop an upper bound on the average sub-optimality gap and characterize the required sample complexity for finding a near-optimal policy. 

Suppose that \Cref{ass:mu(sw)} holds. 
We further adopt a standard assumption also taken by the study of CMDP \citep{levy2022tabular_contextual}. Specifically, given the context $w$ and policy $\pi_w$, let $s_h\sim (P_w,\pi_w)$, and let $p(s_h|\pi_w,P_w)$ denote the density function of $s_h$.
\begin{assumption}\label{ass:pmin} 
For any context $w\in\mathcal{W}$, any step $h\in[H]$ and any context-dependent policy $\pi_w$, there exists constants $0<p_{\mathrm{min}} \leq p_{\mathrm{max}}<\infty$ such that $p(s_h|\pi_w,P_w)\in[p_{\mathrm{min}}, p_{\mathrm{max}}]$ for any $s_h \in  \mathcal{S}$.
\end{assumption}
As discussed in \citet{levy2022tabular_contextual}, the existence of $p_{\mathrm{min}}$ can be removed by mixing each transition kernel with a uniform distribution, while still keeping the sub-optimality gap sublinear without $p_{\mathrm{min}}$. This assumption helps to bound the maximum uncertainty of reaching a state at step $h$ under any context and policy. In this way, the maximum difference between the probability of reaching the current state and reaching any previous state in the history data can be estimated, which allows to guarantee the accuracy of learning the current MDP based on history data.


The following theorem provides an upper bound on the average sub-optimality gap and the number of samples to achieve an $\epsilon$ average sub-optimality gap. We defer the detailed proof of \Cref{thm: model 2} to Appendix~B.

\begin{theorem}\label{thm: model 2}
Consider Model II of CMDP  with varying linear weights as defined in \cref{def:model 2}. Under \Cref{ass:mu(sw),ass:pmin}, for any $\delta\in(0,1)$, with probability at least $1-3\delta/2$, the sequence of policies $\pi_{w_n}^n$ generated by \Cref{alg2} satisfies that
\begin{align*}
&\frac{1}{N}\sum_{n=1}^N\underset{w\sim q}{\mathbb{E}} \left [V^{\pi^*_{w}} _{\mathcal{M}(w)} - V^{\pi^n_{w}} _{\mathcal{M}(w)}\right]  \\
& \textstyle \le912CH^2\sqrt{\frac{d^3K}{N}}\mathrm{log}\left(1+\frac{N}{\widetilde{\lambda}d}\right) +\frac{H^2}{C}\sqrt{\frac{K}{dN}}\left(2\widetilde{\xi}_Nd+C^2\mathrm{log}\left(\frac{2HN|\Psi_3|}{\delta}\right) +4\widetilde{\lambda}_Nd + 8C^2\mathrm{log}\left(\frac{2HN|\Psi_2|}{\delta^2} \right) \right),
\end{align*}
where $\widetilde{\lambda} = \min\{\widetilde{\lambda}_1, \widetilde{\xi}_1\}$ and $C = \sqrt{\frac{p_{\mathrm{max}}}{p_{\mathrm{min}}}}.$ \\
To achieve an $\epsilon$ average sub-optimality gap, at most $\mathcal{O}\left(\frac{H^4d^3K\mathrm{log}^2(|\Psi_2||\Psi_3|/\delta^2)}{\epsilon^2} \cdot \frac{p_{\mathrm{max}}}{p_{\mathrm{min}}}\right)$ episodes are needed.
\end{theorem}
Compared to \Cref{thm:model 1} on Model I, \Cref{thm: model 2} on Model II requires an additional factor of $\mathcal{O}\left( \frac{Kp_{\mathrm{max}}}{p_{\mathrm{min}}}\right)$ in the sample complexity. This is mainly because that the unknown weights are context-varying, and a uniform choice of actions was adopted to facilitate the learning of the varying weights. Such a uniform choice of actions causes an additional factor of $\mathcal{O}(K)$ in the sample complexity.
Moreover, in order to use data collected in the previous contexts to estimate the current MDP, handling the distribution shift among different contexts and policies further introduces an additional sample complexity of order $\mathcal{O}\left( \frac{p_{\mathrm{max}}}{p_{\mathrm{min}}}\right)$.  

The result in \Cref{thm: model 2} improves existing studies as follows. \textbf{(i)} Our result generalizes those results in \citep{amani2022provably_lifilong} to cases where the transition kernel can also be context-varying. \textbf{(ii)} Our Model II includes the tabular CMDP in \citep{levy2022tabular_contextual} as a special case, whereas our model is more general allowing the state space to be infinite. 

\section{Conclusion}
In this paper, we investigated CMDPs whose transition kernel and reward are both context-varying. More specifically, we considered two different linear function approximation models, where in Model I,  both the transition and the reward can be decomposed to a linear function of known context-varying representations and common linear weights; and in Model II, both the transition and the reward can be decomposed to a linear function of known common representations and context-varying linear weights. Both models find their applications in practice. We designed model-based methods for both models, where the design for each model features novel elements to deal with the unique challenge of the model. We further provided provable upper bounds on the average sub-optimality gap for both models and the corresponding sample complexity to achieve $\epsilon$ average sub-optimality gap. 

\medskip
\bibliographystyle{apalike}
\bibliography{ref,icml}


\appendix
\onecolumn
\newpage
\appendix
\centerline{\huge\textbf{Appendix}}\hfill \break
\textbf{Notation.} 
Recall that we use $\mathcal{M}(w)=(\mathcal{S},\mathcal{A},P_w,r_w,H)$ to denote the environment MDP w.r.t. the context $w$. We use $\mathcal{M}^{(r',P')}(w)$ to denote an MDP with a transition kernel $P'_w$ and a reward function $r'_w$, i.e., $\mathcal{M}^{(r',P')}(w)=(\mathcal{S},\mathcal{A},P'_w,r'_w,H)$. Here we define a truncated value function for any MDP $\mathcal M$ under a policy $\pi$ and at step $h$ as:
\begin{align}\label{def:bar}
    &\Bar{Q}^{\pi}_{h,\mathcal{M}^{(r',P')}(w)}(s_h,a_h) = \min\left\{ H,r'_h(s_h,a_h,w)+P'_{h,w}\Bar{V}^{\pi}_{h+1,\mathcal{M}^{(r',P')}(w)}(s_h,a_h)\right\}, \nonumber\\
    &\Bar{V}^{\pi}_{h,\mathcal{M}^{(r',P')}(w)}(s_h) = \underset{\pi}{\mathbb{E}}\left[ \Bar{Q}^{\pi}_{h,\mathcal{M}^{(r',P')}(w)}(s_h,a_h) \right].
\end{align}
For any two probability measures $P$ and $Q$, we use $\cnorm{P-Q}_{TV}$ to denote the total variation distance between them.

\section{Proof of \Cref{thm:model 1}}\label{appd: A}
In this section, we first prove some useful lemmas and then prove \Cref{thm:model 1}.
\subsection{Supporting Lemmas}
We consider the model defined in \Cref{def:model 1}. We first introduce the following MLE guarantee on the estimation error established in \citet{agarwal2020small_pmin}.
\begin{lemma}(MLE guarantee).\label{lemma:model1 mle} Suppose \Cref{ass:mu(s)} holds. Given $\delta \in(0,1)$, we have the following inequality holds with probability at least $1-\delta / 2$ for all $h\in[H]$ and $n\in [N]$:
\begin{equation*}
    \sum_{\tau=1}^{n-1} \underset{\substack{w_\sim q \\ (s_h,a_h)\sim(P
_w,\pi^\tau_{w})}}{\mathbb{E}}\left\|\left\langle\hat{\mu}^n_h(\cdot) - \mu_h(\cdot), \phi_h\left(s_h, a_h, w\right)\right\rangle\right\|_{T V}^2 \le \zeta_n, 
\end{equation*}
where $\zeta_n:= \log (2 |\Psi_1| n H / \delta)$.
\end{lemma}

The following lemma can be obtained from Lemma~39 in \citet{zanette2021cautiously} and Lemma~11 in \citet{uehara2021representation}. 
\begin{lemma}\label{lemma: concentration} 
For $\hat{\Sigma}_h^n$ and $\hat{\Lambda}_h^n$ defined in \cref{eq:hat Lambda}, we define their corresponding expected value as follows:
\[\Sigma^n_h = \sum_{\tau=1}^{n-1} \underset{\substack{w\sim q \\ (s_h,a_h)\sim(P_w,\pi^\tau_{w})}}{\mathbb{E}}\phi_h(s_h,a_h,w)\phi_h(s_h,a_h,w)^\top + \lambda_n I,\]
\[\Lambda^n_h = \sum_{\tau=1}^{n-1} \underset{\substack{w\sim q \\ (s_h,a_h)\sim(P_w,\pi^\tau_{w})}}{\mathbb{E}}\phi_h(s_h,a_h,w)\phi_h(s_h,a_h,w)^\top + \xi_n I,\]
where $\lambda_n = \gamma_1d\mathrm{log}(2nH/\delta)$ and $\xi_n = \gamma_2d\mathrm{log}(2nH/\delta)$. We define the following two events: 
\begin{align*}
\mathcal{E}_1= & \bigg\{ \forall n \in[N], h \in[H], s \in \mathcal{S}, a \in \mathcal{A}, w\in \mathcal{W} \\
& \frac{1}{5}\cnorm{\phi_{h}(s,a,w)}_{\left(\Sigma^n_h\right)^{-1}} \leq \cnorm{\phi_{h}(s,a,w)}_{\left(\hat{\Sigma}^n_h\right)^{-1}} \leq 3\cnorm{\phi_{h}(s,a,w)}_{\left(\Sigma^n_h\right)^{-1}} \bigg\};
\end{align*}
\newpage
\begin{align*}
\mathcal{E}_2= & \bigg\{ \forall n \in[N], h \in[H], s \in \mathcal{S}, a \in \mathcal{A}, w\in \mathcal{W}\\
& \frac{1}{5}\cnorm{\psi_{h}(s,a,w)}_{\left(\Lambda^n_h\right)^{-1}} \leq \cnorm{\psi_{h}(s,a,w)}_{\left(\hat{\Lambda}^n_h\right)^{-1}} \leq 3\cnorm{\psi_{h}(s,a,w)}_{\left(\Lambda^n_h\right)^{-1}} \bigg\}.
\end{align*}
Let $\mathcal{E}_0 := \mathcal{E}_1 \cup \mathcal{E}_2$ denote the intersection of the two events. Then $\mathbb{P}(\mathcal{E}_0) \ge 1-\delta$.
\end{lemma} 

We further prove a number of supporting lemmas.
\begin{lemma}
\label{lemma: linear diff pointwise} Suppose $\hat{P}^n_h(\cdot|s,a,w)=\langle \hat{\mu}^n_h(\cdot),\phi_h(s,a,w) \rangle$ is the estimated context-varying linear transition kernel at step $h\in[H]$ in episode $n\in[N]$. Consider a generic non-negative function $f:\mathcal{S}\rightarrow\mathbb{R}$, which is bounded by $B$, i.e., $f(s) \in [0,B]$ for any $s\in \mathcal{S}$. Then for any $(s,a,w)\in \mathcal{S}\times\mathcal{A}\times\mathcal{W}$, with probability at least $1-\delta/2$, we have:
\begin{equation*}
    \left|\int_{\mathcal{S}} f\left(s'\right)\left(\hat{\mu}^n_h\left(s'\right)-\mu_h\left(s'\right)\right)^{\top} \phi_h(s, a , w) d s'\right| \le \min\left\{\hat{\alpha}_n \cnorm{\phi_h(s,a,w)}_{\left( \Sigma_h^n\right)^{-1}}, B\right\},
\end{equation*}
where $\hat{\alpha}_n = B\sqrt{2\lambda_n d  + 4 \zeta_n}$.
\end{lemma} 

\begin{proof}
First, we have:
\begin{align}
  &\left|\int_{\mathcal{S}} f\left(s'\right)\left(\hat{\mu}^n_h\left(s'\right)-\mu_h\left(s'\right)\right)^{\top} \phi_h(s, a , w) d s'\right| \nonumber\\
  & \qquad \stackrel{(i)}{\leq} \|\phi_h(s, a , w)\|_{\left(\Sigma_h^n\right)^{-1}} \cdot \left\|\int_\mathcal{S} f(s')\left(\hat{\mu}^n_h\left(s'\right)-\mu_h(s')\right) d s' \right\|_{\left(\Sigma_h^n\right)}, \label{eq:step1}
\end{align}
where $(i)$ follows from the Cauchy-Schwarz inequality. Then, we further derive
\begin{align*}
&\hspace{-8mm} \left\|\int_\mathcal{S} f(s)\left(\hat{\mu}^n_h(s)-\mu_h(s)\right) d s\right\|_{\left(\Sigma_h^n\right)}^2 \\
&= \lambda_n \cdot\cnorm{\int_\mathcal{S} f(s)(\hat{\mu}^n_h(s)-\mu_h(s)) d s}^2 \\
&\quad +\sum_{\tau=1}^{n-1} \underset{\substack{w\sim q \\ (s_h,a_h)\sim(P_w,\pi^\tau_{w})}}{\mathbb{E}} \left(\int_\mathcal{S} f(s)\left(\hat{\mu}^n_h(s)-\mu_h(s)\right)^{\top} \phi_h \left(s_h,a_h, w\right) d s\right)^2 \\
&\stackrel{(i)}{\leq}  2\lambda_n d B^2+ 4B^2 \sum_{\tau=1}^{n-1} \underset{\substack{w\sim q \\ (s_h,a_h)\sim(P_w,\pi^\tau_{w})}}{\mathbb{E}}\left\|\left\langle\hat{\mu}^n_h(\cdot) - \mu_h(\cdot), \phi_h\left(s_h, a_h, w\right)\right\rangle\right\|_{T V}^2 \\
&\stackrel{(ii)}{\leq} 2\lambda_n d B^2+ 4B^2 \zeta_n,
\end{align*}
where $(i)$ follows from \Cref{def:model 1} and from the definition of total variation distance, and $(ii)$ follows from \Cref{lemma:model1 mle}. By substituting the above equation into \cref{eq:step1} and setting $\hat{\alpha}_n = B\sqrt{2\lambda_n d  + 4 \zeta_n}$, we have: 
\begin{equation}\label{lemma3: part1}
    \left|\int_{\mathcal{S}} f\left(s'\right)\left(\hat{\mu}^n_h\left(s'\right)-\mu_h\left(s'\right)\right)^{\top} \phi_h(s, a , w) d s'\right| \le \hat{\alpha}_n \cnorm{\phi_h(s,a,w)}_{\left( \Sigma_h^n\right)^{-1}}.
\end{equation}
Also, since $f(s)\in[0,B]$ for any $s\in\mathcal{S}$, we have:
\begin{align}\label{lemma3: part2}
  & \hspace{-1cm}\left|\int_{\mathcal{S}} f\left(s'\right)\left(\hat{\mu}^n_h\left(s'\right)-\mu_h\left(s'\right)\right)^{\top} \phi_h(s, a , w) d s'\right| \nonumber\\
  &  = \left |\underset{s'\sim \hat{P}_h^n(\cdot|s,a,w)}{\mathbb{E}} f(s') - \underset{s'\sim P_h(\cdot|s,a,w)}{\mathbb{E}} f(s') \right |\nonumber\\
  &\leq B.
\end{align}
By combining \cref{lemma3: part1} and \cref{lemma3: part2}, we complete the proof.
\end{proof}

\begin{lemma}
\label{difference of vfunction under different transition} Given the event $\mathcal{E}_0$ defined in \Cref{lemma: concentration} occurs, for any $(s,a,w)\in\mathcal{S}\times\mathcal{A}\times\mathcal{W}$, define the function $\hat{b}_h^n(s,a,w) = \min \{\alpha_n \cnorm{\phi_h(s,a,w)}_{\left( \hat{\Sigma}_h^n\right)^{-1}},H\}$, where $\alpha_n = 5H\sqrt{2\lambda_n d  + 4 \zeta_n}$. Then for any context $w$, with probability at least $1-\delta/2$, we have:
\begin{equation*}
    \left| V_{\mathcal{M}{(w)}}^{\pi_w} - V_{\mathcal{M}^{(r,\hat{P}^n)}(w)}^{\pi_w} \right| \le \Bar{V}_{\mathcal{M}^{(\hat{b}^n,\hat{P}^n)}(w)}^{\pi_w}.
\end{equation*}
\end{lemma} 
\begin{proof}
Recall the definitions of the truncated value functions $\Bar{V}^\pi_{h,\mathcal{M}^{(r',P')}(w)}$ and $\Bar{Q}^\pi_{h,\mathcal{M}^{(r',P')}(w)}$ for a generic MDP $\mathcal{M}^{(r',P')}(w) = (\mathcal{S},\mathcal{A},P'_w,r'_w,H)$ are given by
\begin{align*}
    &\Bar{Q}^{\pi}_{h,\mathcal{M}^{(r',P')}(w)}(s_h,a_h) = \min\{ H,r'(s_h,a_h,w)+P'_{h,w}\Bar{V}^{\pi}_{h+1,\mathcal{M}^{(r',P')}(w)}(s_h,a_h)\}, \\
    &\Bar{V}^{\pi}_{h,\mathcal{M}^{(r',P')}(w)}(s_h) = \underset{\pi}{\mathbb{E}}\left[ \Bar{Q}^{\pi}_{h,\mathcal{M}^{(r',P')}(w)}(s_h,a_h) \right].
\end{align*}
We complete the proof by induction. For the base case $h=H+1$, we have $\left| V_{H+1,\mathcal{M}{(w)}}^{\pi_w}(s_{H+1}) - V_{H+1,\mathcal{M}^{(r,\hat{P}^n)}(w)}^{\pi_w} (s_{H+1})\right| = 0 = \Bar{V}_{H+1,\mathcal{M}^{(\hat{b}^n,\hat{P}^n)}(w)}^{\pi_w}(s_{H+1})$.\\
Now we assume that $\left| V_{h+1,\mathcal{M}{(w)}}^{\pi_w}(s_{h+1}) - V_{h+1,\mathcal{M}^{(r,\hat{P}^n)}(w)}^{\pi_w}(s_{h+1}) \right| \le \Bar{V}_{h+1,\mathcal{M}^{(\hat{b}^n,\hat{P}^n)}(w)}^{\pi_w}(s_{h+1})$ holds for all $s_{h+1}\in \mathcal{S}$. Then according to Bellman equation, for all $s_h,a_h$, we have:
\begin{align}
& \hspace{-1cm}\left| Q_{h,\mathcal{M}^{(r,\hat{P}^n)}{(w)}}^{\pi_w}(s_h,a_h) - Q_{h,\mathcal{M}(w)}^{\pi_w}(s_h,a_h) \right|\nonumber \\
& = \left |\hat{P}^n_{h,w}V^{\pi_w}_{h+1,\mathcal{M}^{(r,\hat{P}^n)}(w)}(s_h,a_h) -{P}_{h,w}V^{\pi_w}_{h+1,\mathcal{M}(w)}(s_h,a_h) \right| \nonumber\\
& = \left |\hat{P}^n_{h,w}\left(V^{\pi_w}_{h+1,\mathcal{M}^{(r,\hat{P}^n)}(w)} - V^{\pi_w}_{h+1,\mathcal{M}(w)}\right)(s_h,a_h) \right. \nonumber\\
&\hspace{1.5in}\left.+\left( \hat{P}^n_{h,w}-P_{h,w}\right)V^{\pi_w}_
{h+1,\mathcal{M}(w)}(s_h,a_h) \right|\nonumber\\
&\stackrel{(i)}{\leq}  \min\left\{H, \hat{b}^n_h(s_h,a_h,w)+\hat{P}^n_{h,w}\left |V^{\pi_w}_{h+1,\mathcal{M}^{(r,\hat{P}^n)}(w)} - V^{\pi_w}_{h+1,\mathcal{M}(w)}\right |(s_h,a_h) \right\}\nonumber\\
&\stackrel{(ii)}{\leq}  \min\left\{H, \hat{b}^n_h(s_h,a_h,w)+\hat{P}^n_{h,w}\Bar{V}^{\pi_w}_{h+1,\mathcal{M}^{(\hat{b}^n,\hat{P}^n)}(w)}(s_h,a_h)\right\}\nonumber\\
& = \Bar{Q}^{\pi_w}_{h,\mathcal{M}^{(\hat{b}^n,\hat{P}^n)}(w)}(s_h,a_h), \label{eq:VrP-VrPhat}
\end{align}
where $(i)$ follows from the fact that $\left| Q_{h,\mathcal{M}^{(r,\hat{P}^n)}{(w)}}^{\pi_w}(s_h,a_h) - Q_{h,\mathcal{M}(w)}^{\pi_w}(s_h,a_h) \right|$ is bounded by $H$ and from \Cref{lemma: linear diff pointwise}, and $(ii)$ follows from the recursive hypothesis. Then, we have:
\begin{align*}
    &\hspace{-1cm}\left| V_{h,\mathcal{M}{(w)}}^{\pi_w}(s_h) - V_{h,\mathcal{M}^{(r,\hat{P}^n)}(w)}^{\pi_w}(s_h) \right| \\
    & = \left | \underset{\pi_w}{\mathbb{E}} \left[ Q^{\pi_w}_{h,\mathcal{M}(w)}(s_h,a_h)\right] - \underset{\pi_w}{\mathbb{E}}\left[Q^{\pi_w}_{h,\mathcal{M}^{(r,\hat{P}^n)}(w)}(s_h,a_h) \right]\right |\\
    & \le \underset{\pi_w}{\mathbb{E}}\left[ \left|Q^{\pi_w}_{h,\mathcal{M}(w)}(s_h,a_h)-Q^{\pi_w}_{h,\mathcal{M}^{(r,\hat{P}^n)}(w)}(s_h,a_h) \right|\right]\\
    & \stackrel{(i)}{\leq} \underset{\pi_w}{\mathbb{E}}\left[ \Bar{Q}^{\pi_w}_{h,\mathcal{M}^{(\hat{b}^n,\hat{P}^n)}(w)}(s_h,a_h)\right]\\
    & = \Bar{V}^{\pi_w}_{h,\mathcal{M}^{(\hat{b}^n,\hat{P}^n)}(w)}(s_h),
\end{align*}
where $(i)$ follows from \cref{eq:VrP-VrPhat}. Therefore, by induction, we conclude that:
\[
    \left| V_{\mathcal{M}{(w)}}^{\pi_w} - V_{\mathcal{M}^{(r,\hat{P}^n)}(w)}^{\pi_w} \right| \le \Bar{V}_{\mathcal{M}^{(\hat{b}^n,\hat{P}^n)}(w)}^{\pi_w}.
\]
\end{proof}

\begin{lemma}
\label{difference of reward}
Suppose $\hat{f}_h^n(s,a,w)$ is the estimated reward defined in \cref{def:hat f}, then for any $(s,a,w)\in \mathcal{S}\times\mathcal{A}\times\mathcal{W}$, we have: \[\left| \hat{f}_h^n(s,a,w) - r_h(s,a,w)\right| \le \min\left\{\beta_n \cnorm{\psi_{h}(s,a,w)}_{\left(\hat{\Lambda}^n_h\right)^{-1}},1\right\},\]
where $\beta_n = \sqrt{\xi_n d}$.
\end{lemma} 

\begin{proof}
Since the values of both $r_h(s,a,w)$ and $\hat{f}_h^n(s,a,w)$ are restricted to $[0,1]$ for any $(s,a,w)$ and any $h\in[H],n\in[N]$, we conclude that $\left| \hat{f}_h^n(s,a,w) - r_h(s,a,w)\right| \le 1$.  Also, according to the definition of $\hat{f}_h^n(s,a,w)$ defined in \cref{def:hat f}, we have:
\begin{equation}\label{lemma5 step1}
    \left| \hat{f}_h^n(s,a,w) - r_h(s,a,w)\right| \le \left| \left\langle\hat{\eta}_h^n, \psi_h(s,a,w)\right\rangle - r_h(s,a,w)\right|.
\end{equation}
Now consider:
\begin{equation*}
\begin{aligned}
\hat{\eta}_h^n-\eta_h & =\left(\hat{\Lambda}_h^n\right)^{-1} \sum_{\tau=1}^{n-1} \psi_h\left(s_h^\tau,a_h^\tau,w_\tau\right) r_h^\tau-\eta_h \\
& = \left(\hat{\Lambda}_h^n\right)^{-1} \left( \sum_{\tau=1}^{n-1} \psi_h\left(s_h^\tau,a_h^\tau,w_\tau\right) r_h^\tau-\hat{\Lambda}_h^n \cdot \eta_h\right) \\
& \stackrel{(i)}{=}-\xi_n\left(\hat{\Lambda}_h^n\right)^{-1}\eta_h,
\end{aligned}
\end{equation*}
where $(i)$ follows from the definition of the matrix $\hat{\Lambda}_h^n$ and the reward function $r_h^\tau = \psi_h\left(s_h^\tau,a_h^\tau,w_\tau\right)^\top \eta_h$. Due to the linear structure, for any $(s,a,w)\in \mathcal{S}\times\mathcal{A}\times\mathcal{W}$, we have:

\begin{align}
\left| \left\langle\hat{\eta}_h^n, \psi_h(s,a,w)\right\rangle - r_h(s,a,w)\right| & = \left|\left\langle\hat{\eta}_h^n-\eta_h, \psi_h(s,a,w)\right\rangle\right|\nonumber\\
&  =\left\langle\left(I-\xi_n\left(\hat{\Lambda}_h^n \right)^{-1}\right) \eta_h-\eta_h, \psi_h (s,a,w) \right\rangle\nonumber\\
&  =\left|\left\langle-\xi_n\left(\hat{\Lambda}_h^n\right)^{-1} \eta_h , \psi_h(s ,a, w)\right\rangle\right| \nonumber\\
&   \leq  \sqrt{\xi_n}\left\|\eta_h\right\| \cdot\left\|\psi_h(s , a, w)\right\|_{\left(\hat{\Lambda}_h^n\right)^{-1}} \nonumber\\
&  \stackrel{(i)}{\leq} \sqrt{\xi_n d}\left\|\psi_h(s,a,w)\right\|_{\left(\hat{\Lambda}_h^n\right)^{-1}}, \label{lemma5 step2}
\end{align}

where $(i)$ follows from the normalization in \Cref{def:model 1}. By combining \cref{lemma5 step1,lemma5 step2}, and the fact that $\left| \hat{f}_h^n(s,a,w) - r_h(s,a,w)\right| \le 1$. Then we complete the proof.
\end{proof}

\begin{lemma}
\label{difference of vfunction under different reward}
Define function $\hat{c}_h^n(s,a,w) = \min\left\{ \beta_n \cnorm{\psi_h(s,a,w)}_{\left( \hat{\Lambda}_h^n\right)^{-1}},1 \right\}$. Assume that the event $\mathcal{E}_0$ defined in \Cref{lemma: concentration} occurs. Then for any context $w$, we have:
\begin{equation*}
    \left| \Bar{V}_{\mathcal{M}^{(\hat{f}^n,\hat{P}^n)}(w)}^{\pi_w} - V_{\mathcal{M}^{(r,\hat{P}^n)}(w)}^{\pi_w} \right| \le \Bar{V}_{\mathcal{M}^{(\hat{c}^n,\hat{P}^n)}(w)}^{\pi_w}.
\end{equation*}
\end{lemma}

\begin{proof}
Note that the values of both $\hat{f}_h^n(s,a,w)$ and $\hat{c}_h^n(s,a,w)$ are restricted to $[0,1]$ for any $(s,a,w)$ and any $h\in[H], n\in[N]$. Following from \Cref{lem:Vbar=V}, it is equivalent to prove:
\begin{equation*}
    \left| {V}_{\mathcal{M}^{(\hat{f}^n,\hat{P}^n)}(w)}^{\pi_w} - V_{\mathcal{M}^{(r,\hat{P}^n)}(w)}^{\pi_w} \right| \le {V}_{\mathcal{M}^{(\hat{c}^n,\hat{P}^n)}(w)}^{\pi_w}.
\end{equation*}
Then we have:
\begin{align}
\left | {V}^{\pi_w} _{\mathcal{M}^{(\hat{f}^n,\hat{P}^n)}(w)}  -V^{\pi_w} _{\mathcal{M}^{(r,\hat{P}^n)}(w)} \right | &  \stackrel{(i)}{=}\left | \sum_{h=1}^H \underset{\left(s_h,a_h\right) \sim \left(\hat{P}_w^n, \pi_w\right)}{\mathbb{E}}\left(\hat{f}_h^n\left(s_h, a_h,w\right)-r_h\left(s_h, a_h,w\right)\right)\right | \nonumber \\
& \le \sum_{h=1}^H \underset{\left(s_h,a_h\right) \sim \left(\hat{P}_w^n, \pi_w\right)}{\mathbb{E}}\left|\hat{f}_h^n\left(s_h, a_h,w\right)-r_h\left(s_h, a_h,w\right)\right|\nonumber \\
& \stackrel{(ii)}{\le} \sum_{h=1}^H \underset{\left(s_h,a_h\right) \sim \left(\hat{P}_w^n, \pi_w\right)}{\mathbb{E}}\hat{c}_h^n(s_h,a_h,w)\nonumber \\
& = {V}_{\mathcal{M}^{(\hat{c}^n,\hat{P}^n)}(w)}^{\pi_w},\label{eq:VhatfhatPhat-VrPhat}
\end{align}
where $(i)$ follows from \Cref{simulation lemma} and $(ii)$ follows from \Cref{difference of reward}.
\end{proof}

\begin{lemma} \label{ lemma:difference of vfunction under different transition and reward} Given the event $\mathcal{E}_0$ occurs, then for any context-dependent policy $\pi_w$ with probability at least $1-\delta/2$, we have
\begin{equation*}
    \left|\Bar{V}^{\pi_w} _{\mathcal{M}^{(\hat{f}^n,\hat{P}^n)}(w)}  -V^{\pi_w} _{\mathcal{M}{(w)}} \right| \le \Bar{V}^{\pi_w} _{\mathcal{M}^{(\hat{b}^n,\hat{P}^n)}(w)} + \Bar{V}^{\pi_w} _{\mathcal{M}^{(\hat{c}^n,\hat{P}^n)}(w)}
\end{equation*}

\end{lemma}
\begin{proof}
By combining the bounds on the estimation error of both the reward and the transition kernel, characterized respectively in \Cref{difference of vfunction under different reward} and \Cref{difference of vfunction under different transition}, we have:

\begin{align*}
\left|\Bar{V}^{\pi_w} _{\mathcal{M}^{(\hat{f}^n,\hat{P}^n)}(w)}  -V^{\pi_w} _{\mathcal{M}{(w)}} \right| = &\left|\Bar{V}^{\pi_w} _{\mathcal{M}^{(\hat{f}^n,\hat{P}^n)}(w)}  -V^{\pi_w} _{\mathcal{M}^{(r,\hat{P}^n)}(w)} 
    +V^{\pi_w} _{\mathcal{M}^{(r,\hat{P}^n)}(w)} -V^{\pi_w} _{\mathcal{M}{(w)}} \right|\\
\le & \left|\Bar{V}^{\pi_w} _{\mathcal{M}^{(\hat{f}^n,\hat{P}^n)}(w)}  -V^{\pi_w} _{\mathcal{M}^{(r,\hat{P}^n)}(w)} \right|
    +\left|V^{\pi_w} _{\mathcal{M}^{(r,\hat{P}^n)}(w)} -V^{\pi_w} _{\mathcal{M}{(w)}} \right|\\
\le & \Bar{V}^{\pi_w} _{\mathcal{M}^{(\hat{b}^n,\hat{P}^n)}(w)} + \Bar{V}^{\pi_w} _{\mathcal{M}^{(\hat{c}^n,\hat{P}^n)}(w)}.
\end{align*}
\end{proof}

\subsection{Proof of \Cref{thm:model 1}}
We first restate \Cref{thm:model 1} below.
\begin{theorem}[Restatement of \Cref{thm:model 1}]
\label{regret} Consider a CMDP with varying representations as defined in \Cref{def:model 1}. Under \Cref{ass:mu(s)}, for any $\delta\in(0,1)$, with probability at least $1-3\delta/2$, the sequence of policies $\{\pi_{w_n}^n\}_{n=1}^N$ generated by \Cref{alg1}  satisfies that
\begin{align*}
\frac{1}{N}\sum_{n=1}^N & \underset{w\sim q}{\mathbb{E}}\left [V^{\pi^*_{w}} _{\mathcal{M}{(w)}} - V^{\pi^n_{w}} _{\mathcal{M}{(w)}}\right]\\
\le & \left(42H\sqrt{2d\lambda_N+4\mathrm{log}(2HN|\Psi_1|/\delta)}+6\sqrt{d\xi_N}\right)H \sqrt{\frac{2d}{N}}  \cdot \sqrt{\mathrm{log}\left(1+\frac{N}{d\lambda} \right)},
\end{align*}

where $\lambda_n = \gamma_1d\mathrm{log}(2nH/\delta)$, $\xi_n = \gamma_2d\mathrm{log}(2nH/\delta)$, $\gamma_1, \gamma_2 = \mathcal{O}(1)$ and $\lambda = \min\{ \lambda_1,\xi_1 \}$. To achieve an $\epsilon$ average sub-optimality gap, at most $\mathcal{O}\left(\frac{H^4d^3\log(|\Psi_1|/\delta)}{\epsilon^2}\right) $ episodes are needed. 
\end{theorem} 
\begin{proof}
First, we derive an optimistic estimation of the optimal value function.
\begin{align}
    V^{\pi^*_{w}} _{\mathcal{M}{({w})}} 
    \stackrel{(i)}{\le} & \Bar{V}^{\pi^*_{w}} _{\mathcal{M}^{(\hat{f}^n,\hat{P}^n)}(w)}+\Bar{V}^{\pi^*_{w}} _{\mathcal{M}^{(\hat{b}^n,\hat{P}^n)}(w)}+\Bar{V}^{\pi^*_{w}} _{\mathcal{M}^{(\hat{c}^n,\hat{P}^n)}(w)}\nonumber\\
    \stackrel{(ii)}{\le} & \doublebar{V}^{\pi^*_{w}} _{\mathcal{M}^{(\hat{f}^n+\hat{b}^n+\hat{c}^n,\hat{P}^n)}(w)}\nonumber\\
    \stackrel{(iii)}{\le} & \doublebar{V}^{\pi^n_{w}} _{\mathcal{M}^{(\hat{f}^n+\hat{b}^n+\hat{c}^n,\hat{P}^n)}(w)}\nonumber \\
    \stackrel{(iv)}{\le} & \doublebar{V}^{\pi^n_{w}} _{\mathcal{M}^{(\hat{f}^n,\hat{P}^n)}(w)}+\doublebar{V}^{\pi^n_{w}} _{\mathcal{M}^{(\hat{b}^n,\hat{P}^n)}(w)}+\doublebar{V}^{\pi^n_{w}} _{\mathcal{M}^{(\hat{c}^n,\hat{P}^n)}(w)}, \label{eq:v under greedy policy}
\end{align}
where $(i)$ follows from \Cref{ lemma:difference of vfunction under different transition and reward}, $(ii)$ follows from \Cref{Vbar<=Vdoublebar}, $(iii)$ follows from the definition of the greed policy $\pi^n_{w}: = \mathrm{argmax}_\pi\doublebar{V}^\pi_{\mathcal{M}^{(\hat{f}^n+\hat{b}^n+\hat{c}^n,\hat{P}^n)}(w)}$ and $(iv)$ follows from \Cref{lem:Vdoublebar<=Vdoublebar}, Then the average suboptimality gap can be bounded by:
\begin{align}
\frac{1}{N}\sum_{n=1}^N & \underset{w\sim q}{\mathbb{E}}\left [V^{\pi^*_{w}} _{\mathcal{M}{(w)}} - V^{\pi^n_{w}} _{\mathcal{M}{(w)}}\right]\nonumber\\
\le &  \frac{1}{N}\sum_{n=1}^N\underset{w\sim q}{\mathbb{E}} \left [  \doublebar{V}^{\pi^n_{w}} _{\mathcal{M}^{(\hat{f}^n,\hat{P}^n)}(w)}+\doublebar{V}^{\pi^n_{w}} _{\mathcal{M}^{(\hat{b}^n,\hat{P}^n)}(w)}+\doublebar{V}^{\pi^n_{w}} _{\mathcal{M}^{(\hat{c}^n,\hat{P}^n)}(w)} - V^{\pi^n_{w}} _{\mathcal{M}{(w)}} \right]\nonumber\\
\le & \frac{1}{N}\sum_{n=1}^N\underset{w\sim q}{\mathbb{E}} \left [ \bigg|\doublebar{V}^{\pi^n_{w}} _{\mathcal{M}^{(\hat{f}^n,\hat{P}^n)}(w)} -  V^{\pi^n_{w}} _{\mathcal{M}{({w})}} \bigg| +  \doublebar{V}^{\pi^n_{w}} _{\mathcal{M}^{(\hat{b}^n,\hat{P}^n)}(w)}+ \doublebar{V}^{\pi^n_{w}} _{\mathcal{M}^{(\hat{c}^n,\hat{P}^n)}(w)}\right]\nonumber\\
\stackrel{(i)}{\le} & \frac{1}{N}\sum_{n=1}^N\underset{w\sim q}{\mathbb{E}} \left [ 2 \doublebar{V}^{\pi^n_{w}} _{\mathcal{M}^{(\hat{b}^n,\hat{P}^n)}(w)}+ 2\doublebar{V}^{\pi^n_{w}} _{\mathcal{M}^{(\hat{c}^n,\hat{P}^n)}(w)}\right], \label{model1step1}
\end{align}
where $(i)$ follows from \Cref{ lemma:difference of vfunction under different transition and reward}.

We next provide an upper bound on $\sum_{n=1}^N\underset{w\sim q}{\mathbb{E}}\doublebar{V}^{\pi^n_{w}} _{\mathcal{M}^{(\hat{b}^n,\hat{P}^n)}(w)}$. Define $g_h^n(s,a,w) := (\hat{P}_{h,w}^n - P_{h,w}) \doublebar{V}^{\pi^n_w}_{h+1,\mathcal{M}^{(\hat{b}^n,\hat{P}^n)}(w)}(s,a)$. Then due to \Cref{lemma: value trans1} in \Cref{app: AUXILIARY LEMMAS}, we have:
\begin{equation}
\sum_{n=1}^N\underset{w\sim q}{\mathbb{E}}\doublebar{V}^{\pi^n_w}_{\mathcal{M}^{(\hat{b}^n,\hat{P}^n)}(w)} \le \sum_{n=1}^N\underset{w\sim q}{\mathbb{E}}V^{\pi^n_w}_{\mathcal{M}^{(\hat{b}^n,P)}(w)}  +  \sum_{n=1}^N\underset{w\sim q}{\mathbb{E}}V^{\pi^n_w}_{\mathcal{M}^{(g^n,P)}(w)}.\label{thm value trans1}
\end{equation}

First, we bound the first term in the right-hand-side of \cref{thm value trans1}. By applying Lemma \ref{potential lemma} in \Cref{app: AUXILIARY LEMMAS}, we can obtain a bound on the summation of the expected value function $V^{\pi_w}_{\mathcal{M}^{(\hat{b}^n,P)}(w)}$ as follows:
\begin{align}
\sum_{n=1}^N\underset{w\sim q}{\mathbb{E}} {V}^{\pi^n_{w}} _{\mathcal{M}^{(\hat{b}^n,P)}(w)} 
= & \sum_{n=1}^N \sum _{h=1}^H  
\underset{\substack{w \sim q\\(s_h, a_h) \sim (P_{w},\pi ^n_{w})}}{\mathbb{E}} \left[ \alpha_n \cnorm{\phi_h (s_h,a_h, w)}_{(\hat{\Sigma}_h^n)^{-1}} \right] \nonumber \\
\stackrel{(i)}{\le} & 3\alpha_N\sqrt{N} \sum_{h=1}^H \sqrt{ \sum _{n =1}^N \underset{\substack{w \sim q \\ (s_h, a_h) \sim (P_{w},\pi ^n_{w})}}{\mathbb{E}} \left[ \cnorm{\phi_h (s_h,a_h, w)}_{(\Sigma_h^n)^{-1}}^2\right] } \nonumber\\
\stackrel{(ii)}{\le} & 3\alpha_N \sqrt{N} H \cdot \sqrt{2d \mathrm{log}\left(1+\frac{N}{d\lambda} \right)},\label{thm value trans2}
\end{align}
where $(i)$ follows from the Cauchy-Schwarz inequality and because the event $\mathcal{E}_0$ occurs, and $(ii)$ follows from \Cref{potential lemma}. 

Next, we bound the second term in the right-hand-side of \cref{thm value trans1}. We obtain the bound on the summation of the expected value function $V^{\pi_w}_{\mathcal{M}^{(g^n,P )}(w)}$ as follows:
\begin{align}
\sum_{n=1}^N\underset{w\sim q}{\mathbb{E}} V^{\pi^n_{w}} _{\mathcal{M}^{(g^n,P)}(w)} 
\le & \sum_{n=1}^N \sum _{h=1}^H  
\underset{\substack{w \sim q\\(s_h, a_h) \sim (P_{w},\pi ^n_{w})}}{\mathbb{E}} \big |g_h^n(s_h,a_h,w) \big| \nonumber\\
\stackrel{(i)}{\le} & \sum_{n=1}^N \sum _{h=1}^H  
\underset{\substack{w \sim q\\(s_h, a_h) \sim (P_{w},\pi ^n_{w})}}{\mathbb{E}} \left[ \frac{3}{5}\alpha_n \cnorm{\phi_h (s_h,a_h, w)}_{(\Sigma_h^n)^{-1}} \right] \nonumber\\
\stackrel{(ii)}{\le} & \frac{3}{5}\alpha_N\sqrt{N} \sum_{h=1}^H \sqrt{ \sum _{n =1}^N \underset{\substack{w \sim q \\ (s_h, a_h) \sim (P _{w},\pi ^n_{w})}}{\mathbb{E}} \left[ \cnorm{\phi_h (s_h,a_h, w)}_{(\Sigma_h)^{-1}}^2\right] } \nonumber\\
\stackrel{(iii)}{\le} & \frac{3}{5}\alpha_N \sqrt{N} H \cdot \sqrt{2d \mathrm{log}\left(1+\frac{N}{d\lambda} \right)},\label{thm value trans3}
\end{align}
where $(i)$ follows from \Cref{lemma: linear diff pointwise} and the fact that $\doublebar{V}^{\pi_w}_{h,\mathcal{M}^{(\hat{b}^n,\hat{P}^n)}(w)}(s_h,a_h) \le 3H$ for any $h\in[H]$, $(ii)$ follows from the Cauchy-Schwarz inequality, and $(iii)$ follows from \Cref{potential lemma}. Then by combining \cref{thm value trans1,thm value trans2,thm value trans3} together with the definition of $\alpha_n$ we obtain:
\begin{equation}\label{thm value bound1}
    \sum_{n=1}^N\underset{w\sim q}{\mathbb{E}} \doublebar{V}^{\pi^n_{w}} _{\mathcal{M}^{(\hat{b}^n,\hat{P}^n)}(w)} \le \frac{18}{5}\alpha_N \sqrt{N} H \cdot \sqrt{2d \mathrm{log}\left(1+\frac{N}{d\lambda} \right)}.
\end{equation}

Now, we provide an upper bound on $\sum_{n=1}^N\underset{w\sim q}{\mathbb{E}}\doublebar{V}^{\pi^n_w}_{\mathcal{M}^{(\hat{c}^n,\hat{P}^n)}(w)}$. Due to \Cref{lemma: value trans1} in \Cref{app: AUXILIARY LEMMAS}, we can show that
\begin{equation} \label{thm value reward1}
\sum_{n=1}^N\underset{w\sim q}{\mathbb{E}}\doublebar{V}^{\pi^n_w}_{\mathcal{M}^{(\hat{c}^n,\hat{P}^n)}(w)} \le \sum_{n=1}^N\underset{w\sim q}{\mathbb{E}}V^{\pi^n_w}_{\mathcal{M}^{(\hat{c}^n,P)}(w)} + \sum_{n=1}^N\underset{w\sim q}{\mathbb{E}}V^{\pi^n_w}_{\mathcal{M}^{(l^n,P)}(w)},
\end{equation}
where $l_h^n(s,a,w) := (\hat{P}_{h,w}^n - P_{h,w}) \doublebar{V}^{\pi^n_w}_{h+1,\mathcal{M}^{(\hat{c}^n,\hat{P}^n)}(w)}(s,a)$.
We first provide an upper bound on the first term in the right-hand-side of \cref{thm value reward1}:
\begin{align}
\sum_{n=1}^N\underset{w\sim q}{\mathbb{E}} V^{\pi^n_{w}} _{\mathcal{M}^{(\hat{c}^n,P)}(w)} 
= & \sum_{n=1}^N \sum _{h=1}^H  
\underset{\substack{w \sim q\\(s_h, a_h) \sim (P_{w},\pi ^n_{w})}}{\mathbb{E}} \left[ \beta_n \cnorm{\psi_h (s_h,a_h, w)}_{(\hat{\Lambda}_h^n)^{-1}} \right] \nonumber\\
\stackrel{(i)}{\le} & 3\beta_N\sqrt{N} \sum_{h=1}^H \sqrt{ \sum _{n =1}^N \underset{\substack{w \sim q \\ (s_h, a_h) \sim (P_{w},\pi ^n_{w})}}{\mathbb{E}} \left[ \cnorm{\psi_h (s_h,a_h, w)}_{(\Lambda_h^n)^{-1}}^2\right] } \nonumber\\
\stackrel{(ii)}{\le} & 3\beta_N \sqrt{N} H \cdot \sqrt{2d \mathrm{log}\left(1+\frac{N}{d\lambda} \right)},\label{thm value reward2}
\end{align}
where $(i)$ follows from the Cauchy-Schwarz inequality and the event $\mathcal{E}_0$ occurs, and $(ii)$ follows from \Cref{potential lemma}. Then, since $\doublebar{V}^{\pi_w}_{h,\mathcal{M}^{(\hat{c}^n,\hat{P}^n)}(w)}(s,a) \le 3H$ for any $h\in[H]$, we bound the second term in the right-hand-side of \cref{thm value reward1} similarly to \cref{thm value trans3} and obtain:
\begin{align}
\sum_{n=1}^N\underset{w\sim q}{\mathbb{E}} V^{\pi^n_{w}} _{\mathcal{M}^{(l^n,P)}(w)} \le \frac{3}{5}\alpha_N \sqrt{N} H \cdot \sqrt{2d \mathrm{log}\left(1+\frac{N}{d\lambda} \right)}.\label{thm value trans4}
\end{align}

Combining \cref{thm value trans4,thm value reward1,thm value reward2} together with the definitions of $\alpha_n$ and $\beta_n$, we obtain:
\begin{equation}\label{thm value bound2}
    \sum_{n=1}^N\underset{w\sim q}{\mathbb{E}} \doublebar{V}^{\pi^n_{w}} _{\mathcal{M}^{(\hat{c}^n,\hat{P}^n)}(w)} \le \left(\frac{3}{5}\alpha_N+3\beta_N\right) \sqrt{N} H \cdot \sqrt{2d \mathrm{log}\left(1+\frac{N}{d\lambda} \right)}.
\end{equation}
By substituting \cref{thm value bound1,thm value bound2} into \cref{model1step1}, we can bound the average suboptimality gap as follows:
\[
\begin{aligned}
\frac{1}{N}\sum_{n=1}^N & \underset{w\sim q}{\mathbb{E}}\left [V^{\pi^*_{w}} _{\mathcal{M}{(w)}} - V^{\pi^n_{w}} _{\mathcal{M}{(w)}}\right]\\
\le & \frac{1}{N}\left(\frac{42}{5}\alpha_N+6\beta_N\right) \sqrt{N} H \cdot \sqrt{2d \mathrm{log}\left(1+\frac{N}{d\lambda} \right)}\\
= & \left(42H\sqrt{2d\lambda_N+4\mathrm{log}(2HN|\Psi_1|/\delta)}+6\sqrt{d\xi_N}\right)H \sqrt{\frac{2d}{N}}  \cdot \sqrt{\mathrm{log}\left(1+\frac{N}{d\lambda} \right)},
\end{aligned}
\]
where $\lambda_n = \gamma_1d\mathrm{log}(2nH/\delta)$, $\xi_n = \gamma_2d\mathrm{log}(2nH/\delta)$ and $\lambda = \min\{ \lambda_1,\xi_1 \}$. 

Note that $\lambda_N,\xi_N = \mathcal{O}(d)$, and it can be seen that the upper bound of the average suboptimality gap is of the order $\mathcal{O}\left(\sqrt{\frac{H^4d^3\log(|\Psi_1|/\delta)}{N}}\right)$. Then to achieve an $\epsilon$ average sub-optimality gap, at most $\mathcal{O}\left(\frac{H^4d^3\log(|\Psi_1|/\delta)}{\epsilon^2}\right) $ episodes are needed. This completes the proof.
\end{proof}

\section{Proof of \Cref{thm: model 2}. }\label{appd: B}
In this section, we first prove some useful lemmas and then prove \Cref{thm: model 2}.
\subsection{Supporting Lemmas}
We consider the model defined in \Cref{def:model 2}. We first introduce the following MLE guarantee on the estimation error established in \citet{agarwal2020small_pmin}. Note that the form of the MLE guarantee is different from that in \Cref{lemma:model1 mle} because the model is different.
\begin{lemma}(MLE guarantee).\label{lemma:model2 mle} Suppose \Cref{ass:mu(sw)} holds. Given $\delta \in(0,1)$, the following inequality holds with probability at least $1-\delta / 2$ for all $h\in[H]$ and $n\in [N]$:
\begin{equation*}
    \sum_{\tau=1}^{n-1} \underset{\substack{w_\sim q \\ s_h\sim(P_w,\pi^\tau_{w}) \\ a_h \sim \mathcal{U}(\mathcal{A})}}{\mathbb{E}}\left\|\left\langle\widetilde{\mu}_h(\cdot,w) - \mu_h(\cdot,w), \phi_h\left(s_h, a_h\right)\right\rangle\right\|_{T V}^2 \le \zeta_n, 
\end{equation*}
where $\zeta_n:=\log (2 |\Psi_2| n H / \delta)$.
\end{lemma} 

\begin{lemma} (LSR guarantee). \label{lemma:model2 lsr} Suppose \Cref{ass:mu(sw)} holds. Given $\delta \in(0,1)$, the following inequality holds with probability at least $1-\delta / 2$ for all $h\in[H]$ and $n\in [N]$:
\begin{equation*}
    \sum_{\tau=1}^{n-1} \underset{\substack{w_\sim q \\ s_h\sim(P_w,\pi^\tau_{w}) \\ a_h \sim \mathcal{U}(\mathcal{A})}}{\mathbb{E}}\left\|\left\langle\widetilde{\eta}_h(w) - \eta_h(w), \psi_h\left(s_h, a_h\right)\right\rangle\right\|_2^2 \le \zeta'_n, 
\end{equation*}
where $\zeta'_n:=\log (2 |\Psi_3| n H / \delta)$.
\end{lemma}
\begin{proof}
We present the detailed proof in \Cref{app: LSR guarantee}.
\end{proof}

The following lemma can be obtained from Lemma~39 in \citet{zanette2021cautiously} and Lemma~11 in \citet{uehara2021representation}.
\begin{lemma}
For $\widetilde{\Sigma}_h^n$ defined in \cref{widSigma}, we define its expected value as follows:
\begin{align*}
\Sigma^n_h = \sum_{\tau=1}^{n-1} \underset{\substack{w\sim q \\ s_h\sim(P_w,\pi^\tau_{w}) \\ a_h \sim \mathcal{U}(\mathcal{A})}}{\mathbb{E}}\phi_h(s_h,a_h)\phi_h(s_h,a_h)^\top + \widetilde\lambda_n I,\\
\Lambda^n_h = \sum_{\tau=1}^{n-1} \underset{\substack{w\sim q \\ s_h\sim(P_w,\pi^\tau_{w}) \\ a_h \sim \mathcal{U}(\mathcal{A})}}{\mathbb{E}}\psi_h(s_h,a_h)\psi_h(s_h,a_h)^\top + \widetilde\xi_n I
\end{align*}
where $\widetilde\lambda_n = \widetilde{\gamma}_1d\mathrm{log}(2nH/\delta)$ and $\widetilde\xi_n = \widetilde{\gamma}_2d\mathrm{log}(2nH/\delta)$. We further define $\widetilde{\mathcal{E}}_0 = \widetilde{\mathcal{E}}_1\cup \widetilde{\mathcal{E}}_2$ where:
\begin{align*}
\widetilde{\mathcal{E}}_1= \bigg\{ &\forall n \in[N], h \in[H], s \in \mathcal{S}, a \in \mathcal{A},\\
& \frac{1}{5}\cnorm{\phi_{h}(s,a)}_{\left(\Sigma^n_h\right)^{-1}} \leq \cnorm{\phi_{h}(s,a)}_{\left(\widetilde{\Sigma}^n_h\right)^{-1}} \leq 3\cnorm{\phi_{h}(s,a)}_{\left(\Sigma^n_h\right)^{-1}} \bigg\},\\
\widetilde{\mathcal{E}}_2= \bigg\{ &\forall n \in[N], h \in[H], s \in \mathcal{S}, a \in \mathcal{A},\\
& \frac{1}{5}\cnorm{\psi_{h}(s,a)}_{\left(\Lambda^n_h\right)^{-1}} \leq \cnorm{\psi_{h}(s,a)}_{\left(\widetilde{\Lambda}^n_h\right)^{-1}} \leq 3\cnorm{\psi_{h}(s,a)}_{\left(\Lambda^n_h\right)^{-1}} \bigg\}.
\end{align*}
Then we have $\mathbb{P}(\widetilde{\mathcal{E}}_0) \ge 1-\delta$.
\end{lemma} 

We next prove a number of supporting lemmas that are useful for our proof. 
\begin{lemma}\label{lemma:switch context}
Suppose \Cref{ass:pmin} holds. Given any function $f:\mathcal{S}\times\mathcal{A}\times\mathcal{W} \rightarrow \mathbb{R}$, for any given context $w\in\mathcal{W}$, we have:
\begin{equation*}
\underset{\substack{w_\tau \sim q \\ s_h\sim(P_{w_\tau},\pi^\tau_{w_\tau}) \\ a_h \sim \mathcal{U}(\mathcal{A})}}{\mathbb{E}} f (s_h,a_h,w) \le C^2
\underset{\substack{ s_h\sim(P_{w},\pi^\tau_{w}) \\ a_h \sim \mathcal{U}(\mathcal{A})}}{\mathbb{E}} f (s_h,a_h,w),
\end{equation*}
where $C = \sqrt{\frac{p_{\mathrm{max}}}{p_{\mathrm{min}}}}$.
\end{lemma}
\begin{proof}
Recall that we use $p(s_h|\pi_w,P_w)$ to denote the probability density of $s_h$ when $s_h\sim (P_w,\pi_w)$. For any given $w_\tau\in \mathcal{W}$, we have:
\begin{equation*}
\begin{aligned}
    \underset{\substack{s_h\sim(P_{w_\tau},\pi^\tau_{w_\tau}) \\ a_h \sim \mathcal{U}(\mathcal{A})}}{\mathbb{E}} f (s_h,a_h,w) = &\int _{\mathcal{S}} \sum_{a_h} f(s_h,a_h,w) \cdot p(s_h|\pi_{w_\tau},P_{w_\tau})\cdot \frac{1}{K} d s_h\\
    \stackrel{(i)}{\le} & \frac{p_{\mathrm{max}}}{p_{\mathrm{min}}}\int _{\mathcal{S}} \sum_{a_h} f(s_h,a_h,w) \cdot p(s_h|\pi_w,P_w)\cdot \frac{1}{K} ds_h\\
    = & C^2 \underset{\substack{s_h\sim(P_{w},\pi^\tau_{w}) \\ a_h \sim \mathcal{U}(\mathcal{A})}}{\mathbb{E}} f (s_h,a_h,w),
\end{aligned}
\end{equation*}
where $(i)$ follows from \Cref{ass:pmin}. Note that the right-hand-side of the above equation is independent of $w_\tau$. By taking the expectation over $w_\tau$ on both sides, we obtain the desired result.
\end{proof}
To simplify the notation, we define \begin{equation}\label{def:wid zeta}
    \widetilde{\zeta}_h^n(w) = \sum_{\tau=1}^{n-1} \underset{\substack{s_h\sim(P_{w},\pi^\tau_{w}) \\ a_h \sim \mathcal{U}(\mathcal{A})}}{\mathbb{E}}\left\|\left\langle\widetilde{\mu}^n_h(\cdot,w) - \mu_h(\cdot,w), \phi_h\left(s_h, a_h\right)\right\rangle\right\|_{T V}^2. 
\end{equation}
Now we present a lemma to bound the value function gap $|(P_{w,h}-\widetilde{P}^n_{w,h}) V_{w,h+1}(s,a)|$. Differently from \Cref{lemma: linear diff pointwise} in \Cref{appd: A}, we cannot bound the context-varying gap by the norms of context-independent representations as bonus terms for any $(s,a,w)$. The following lemma helps to decompose the context-varying value function gap into context-dependent and context-independent components.
\begin{lemma}\label{lemma:linear diff expect}
Suppose $\widetilde{P}^n_h(\cdot|s_h,a_h,w)=\langle \widetilde{\mu}^n_h(\cdot,w),\phi_h(s_h,a_h) \rangle$ is the estimated context-varying linear transition kernel at step $h\in[H]$ in episode $n\in[N]$. Consider a generic non-negative function $f:\mathcal{S}\rightarrow\mathbb{R}$ which is bounded by $B$, i.e., $f(s) \in [0,B]$ for any $s\in\mathcal{S}$. Then any $(s_h,a_h,w)\in \mathcal{S}\times\mathcal{A}\times\mathcal{W}$, with probability at least $1-\delta/2$, we have:
\begin{align*}
    &\hspace{-1cm} \left|\int_{\mathcal{S}} f\left(s'\right)\left(\widetilde{\mu}^n_h\left(s',w\right)-\mu_h\left(s',w\right)\right)^{\top} \phi_h(s_h,a_h) d s'\right|\\
    & \le \min \left\{ \frac{CB\sqrt{dN}}{2\sqrt{K}} \cnorm{\phi_h(s_h,a_h)}_{(\Sigma_h^n)^{-1}}^2, B\right\} +\frac{B\sqrt{K}}{C\sqrt{dN}}(\widetilde{\lambda}_nd+2C^2\widetilde{\zeta}_h^n(w)).
\end{align*}
\end{lemma}
\begin{proof}
Consider any given $(s_h,a_h,w)\in\mathcal{S}\times \mathcal{A}\times \mathcal{W}$. We first obtain:
\begin{align}
& \left|\int_{\mathcal{S}} f(s')(\widetilde{\mu}^n_h(s',w)-\mu_h(s',w))^{\top} \phi_h(s_h,a_h) d s'\right|  \nonumber\\
& \hspace{8mm} \stackrel{(i)}{\leq}  \|\phi_h(s_h,a_h)\|_{\left(\Sigma_h^n\right)^{-1}} \cdot\left\|\int_\mathcal{S} f(s')\left(\widetilde{\mu}^n_h\left(s',w\right)-\mu_h(s',w)\right) d s'  \right\|_{\left(\Sigma_h^n\right)},\label{eq:linear diff 1}
\end{align}

where $(i)$ follows from the Cauchy-Schwarz inequality. Then, we further derive that
\begin{align}\label{eq:linear diff 2}
& \hspace{-1cm}\left\|\int_\mathcal{S} f(s)\left(\widetilde{\mu}^n_h(s,w)-\mu_h(s,w)\right) d s\right\|_{\left(\Sigma_h^n\right)}^2 \nonumber\\
&= \widetilde\lambda_n \cdot\cnorm{\int_\mathcal{S} f(s)(\widetilde{\mu}^n_h(s,w)-\mu_h(s,w)) d s}^2 \nonumber\\ 
& \hspace{9mm} +\sum_{\tau=1}^{n-1} \underset{\substack{w_\tau \sim q \\ s'_h\sim(P_{w_\tau},\pi^\tau_{w_\tau}) \\ a'_h \sim \mathcal{U}(\mathcal{A})}}{\mathbb{E}}\left(\int_\mathcal{S} f(s)\left(\widetilde{\mu}^n_h(s,w)-\mu_h(s,w)\right)^{\top} \phi_h \left(s'_h,a'_h\right) d s\right)^2 \nonumber\\
&\stackrel{(i)}{\leq}  2\lambda_n d B^2+4B^2 \sum_{\tau=1}^{n-1} \underset{\substack{w_\tau \sim q \\ s'_h\sim(P_{w_\tau},\pi^\tau_{w_\tau}) \\ a'_h \sim \mathcal{U}(\mathcal{A})}}{\mathbb{E}} \left\|\left\langle\widetilde{\mu}^n_h(\cdot,w) - \mu_h(\cdot,w), \phi_h\left(s'_h, a'_h\right)\right\rangle\right\|_{T V}^2 \nonumber\\
&\stackrel{(ii)}{\leq}  B^2 \left(2\lambda_n d +4 C^2 \widetilde{\zeta}_h^n(w)\right),
\end{align}

where $(i)$ follow from \Cref{def:model 2} and from the definition of the total variation distance, and $(ii)$ follows from \Cref{lemma:switch context} and the notation \cref{def:wid zeta}. 
Notice that since $f(s)\in[0,B]$ for any $s\in\mathcal{S}$, we have:
\begin{align}\label{lemma12: part2}
  & \hspace{-1cm}\left|\int_{\mathcal{S}} f\left(s'\right)\left(\widetilde{\mu}^n_h\left(s'\right)-\mu_h\left(s'\right)\right)^{\top} \phi_h(s, a , w) d s'\right| \nonumber\\
  &  = \left |\underset{s'\sim \widetilde{P}_h^n(\cdot|s,a,w)}{\mathbb{E}} f(s') - \underset{s'\sim P_h(\cdot|s,a,w)}{\mathbb{E}} f(s') \right | \leq B.
\end{align}
Then, we have:
\begin{align*}
&\hspace{-1cm} \left|\int_{\mathcal{S}}f\left(s'\right)\left(\widetilde{\mu}^n_h\left(s',w\right)-\mu_h\left(s',w\right)\right)^{\top} \phi_h(s_h,a_h) d s'\right|  \\
& \stackrel{(i)}{\le} \min\left\{ \|\phi_h(s_h,a_h)\|_{\left(\Sigma_h^n\right)^{-1}} \cdot \sqrt{B^2 \left(2\lambda_n d +4 C^2 \widetilde{\zeta}_h^n(w)\right)},B \right\}\\
&=  \min\left\{\sqrt{\frac{CB\sqrt{dN}}{\sqrt{K}} \cnorm{\phi_h(s_h,a_h)}_{(\Sigma_h^n)^{-1}}^2 } \cdot \sqrt{\frac{B\sqrt{K}}{C\sqrt{dN}}(2\widetilde{\lambda}_nd+4C^2\widetilde{\zeta}_h^n(w))},B \right\}\\
&\stackrel{(ii)}{\le} \min\left\{ \frac{CB\sqrt{dN}}{2\sqrt{K}} \cnorm{\phi_h(s_h,a_h)}_{(\Sigma_h^n)^{-1}}^2 +\frac{B\sqrt{K}}{C\sqrt{dN}}(\widetilde{\lambda}_nd+2C^2\widetilde{\zeta}_h^n(w)),B \right\} \\
&\stackrel{(iii)}{\le} \min\left\{ \frac{CB\sqrt{dN}}{2\sqrt{K}} \cnorm{\phi_h(s_h,a_h)}_{(\Sigma_h^n)^{-1}}^2,B \right\} +\frac{B\sqrt{K}}{C\sqrt{dN}}(\widetilde{\lambda}_nd+2C^2\widetilde{\zeta}_h^n(w)),
\end{align*}

where $(i)$ follows from \cref{eq:linear diff 1,eq:linear diff 2,lemma12: part2}, $(ii)$ follows from the fact that $ab\le \frac{1}{2}a^2+\frac{1}{2}b^2$ and $(iii)$ follows from the fact that $\min\{a,b+c\} \le \min\{a,b\}+\min\{a,c\}$ for $a,b,c\ge 0$.
\end{proof}

\begin{lemma}
\label{lemma: difference of vfunction under different transition_2} Given that the event $\widetilde{\mathcal{E}}_0$ occurs, for any $(s,a)\in\mathcal{S}\times\mathcal{A}$, define the function $\widetilde{b}_h^n(s,a) := \min \left\{\widetilde\alpha_n \cnorm{\phi_h(s,a)}^2_{\left( \widetilde{\Sigma}_h^n\right)^{-1}},H \right\}$, where $\widetilde\alpha_n = \frac{25CH\sqrt{dN}}{2\sqrt{K}}$. Then for any context-dependent policy $\pi_w$, with probability at least $1-\delta/2$, we have:
\begin{equation*}
    \underset{\substack{w \sim q}}{\mathbb{E}} \left| V_{\mathcal{M}{(w)}}^{\pi_w} - V_{\mathcal{M}^{(r,\widetilde{P}^n)}(w)}^{\pi_w} \right| \le \underset{\substack{w \sim q}}{\mathbb{E}} \Bar{V}_{\mathcal{M}^{(\widetilde{b}^n,\widetilde{P}^n)}(w)}^{\pi_w} + \frac{H^2\sqrt{K}}{C\sqrt{dN}} (\widetilde{\lambda}_nd+2C^2\zeta_n).
\end{equation*}
\end{lemma} 
\begin{proof}
Recall the truncated value functions $\Bar{V}^\pi_{h,\mathcal{M}^{(r',P')}(w)}$ and $\Bar{Q}^\pi_{h,\mathcal{M}^{(r',P')}(w)}$ for a generic MDP $\mathcal{M}^{(r',P')}(w) = (\mathcal{S},\mathcal{A},P'_w,r'_w,H)$ are defined as:
\begin{align*}
&\Bar{Q}^{\pi}_{h,\mathcal{M}^{(r',P')}(w)}(s_h,a_h) = \min\{ H,r'(s_h,a_h,w)+P'_{h,w}\Bar{V}^{\pi}_{h+1,\mathcal{M}^{(r',P')}(w)}(s_h,a_h)\}, \\
    &\Bar{V}^{\pi}_{h,\mathcal{M}^{(r',P')}(w)}(s_h) = \underset{\pi}{\mathbb{E}}\left[ \Bar{Q}^{\pi}_{h,\mathcal{M}^{(r',P')}(w)}(s_h,a_h) \right].
\end{align*}
We first prove that 
\[
    \left| V_{\mathcal{M}{(w)}}^{\pi_w} - V_{\mathcal{M}^{(r,\widetilde{P}^n)}(w)}^{\pi_w} \right| \le \Bar{V}_{\mathcal{M}^{(\widetilde{b}^n,\widetilde{P}^n)}(w)}^{\pi_w} + \frac{H \sqrt{K}}{C\sqrt{dN}}  \sum_{h=1}^H (\widetilde{\lambda}_nd+2C^2\widetilde{\zeta}_h^n(w)),
\] holds by induction. For the base case $h=H$, we have \begin{align*}
    &\hspace{-1cm} \left| V_{H,\mathcal{M}{(w)}}^{\pi_w}(s_H) - V_{H,\mathcal{M}^{(r,\widetilde{P}^n)}(w)}^{\pi_w} (s_H)\right| \\
    & \stackrel{(i)}{\le} \underset{\substack{a_{H}\sim \pi_w}}{\mathbb{E}}\left| \left(P_{H,w}-\widetilde{P}^n_{H,w}\right) V_{H+1, \mathcal{M}(w)}^{\pi_w}\left(s_H, a_H\right)\right| \\
    & \stackrel{(ii)}{\le}  \underset{\substack{a_{H}\sim \pi_w}}{\mathbb{E}}\left[\min \left\{H, \frac{CH\sqrt{dN}}{2\sqrt{K}} \cnorm{\phi_H(s_H,a_H)}_{(\Sigma_H^n)^{-1}}^2 \right\} \right] +\frac{H\sqrt{K}}{C\sqrt{dN}}(\widetilde{\lambda}_nd+2C^2\widetilde{\zeta}_H^n(w)) \\
    & = \Bar{V}_{H,\mathcal{M}^{(\widetilde{b}^n,\widetilde{P}^n)}(w)}^{\pi_w}(s_H) + \frac{H\sqrt{K}}{C\sqrt{dN}} (\widetilde{\lambda}_nd+2C^2\widetilde{\zeta}_H^n(w)),
\end{align*}
where $(i)$ follows from \Cref{simulation lemma} and $(ii)$ follows from \Cref{lemma:linear diff expect}.

Now we assume that \begin{align*}
    &\hspace{-1cm} \left| V_{h+1,\mathcal{M}{(w)}}^{\pi_w}(s_{h+1}) - V_{h+1,\mathcal{M}^{(r,\widetilde{P}^n)}(w)}^{\pi_w} (s_{h+1})\right| \\
    & \le \Bar{V}_{h+1,\mathcal{M}^{(\widetilde{b}^n,\widetilde{P}^n)}(w)}^{\pi_w}(s_{h+1}) + \sum_{k=h+1}^H \frac{H\sqrt{K}}{C\sqrt{dN}} (\widetilde{\lambda}_nd+2C^2\widetilde{\zeta}_k^n(w))
\end{align*} holds for all $s_{h+1}\in \mathcal{S}$. Then following from the Bellman equation, for all $s_h,a_h$, we have:
\begin{align}\label{lemma13 step1}
& \hspace{-1cm}\left| Q_{h,\mathcal{M}^{(r,\widetilde{P}^n)}{(w)}}^{\pi_w}(s_h,a_h) - Q_{h,\mathcal{M}(w)}^{\pi_w}(s_h,a_h) \right|\nonumber \\
& = \left |\widetilde{P}^n_{h,w}V^{\pi_w}_{h+1,\mathcal{M}^{(r,\widetilde{P}^n)}(w)}(s_h,a_h) -{P}_{h,w}V^{\pi_w}_{h+1,\mathcal{M}(w)}(s_h,a_h) \right| \nonumber\\
& = \left|\widetilde{P}^n_{h,w}\left(V^{\pi_w}_{h+1,\mathcal{M}^{(r,\widetilde{P}^n)}(w)} - V^{\pi_w}_{h+1,\mathcal{M}(w)}\right)(s_h,a_h) \right. \nonumber\\
&\hspace{1.5in}\left.+\left( \widetilde{P}^n_{h,w}-P_{h,w}\right)V^{\pi_w}_
{h+1,\mathcal{M}(w)}(s_h,a_h) \right|\nonumber\\
& \le \left |\widetilde{P}^n_{h,w}\left(V^{\pi_w}_{h+1,\mathcal{M}^{(r,\widetilde{P}^n)}(w)} - V^{\pi_w}_{h+1,\mathcal{M}(w)}\right)(s_h,a_h) \right| \nonumber\\
&\hspace{1.5in}+ \left|\left( \widetilde{P}^n_{h,w}-P_{h,w}\right)V^{\pi_w}_
{h+1,\mathcal{M}(w)}(s_h,a_h) \right|. 
\end{align}
Now we consider the first term in \cref{lemma13 step1}:
\begin{align} \label{lemma13 step2}
&\hspace{-1cm}\left |\widetilde{P}_{h,w}\left(V^{\pi_w}_{h+1,\mathcal{M}^{(r,\widetilde{P}^n)}(w)} - V^{\pi_w}_{h+1,\mathcal{M}(w)}\right)(s_h,a_h) \right|\nonumber \\
& \le \widetilde{P}_{h,w}\left|V^{\pi_w}_{h+1,\mathcal{M}^{(r,\widetilde{P}^n)}(w)} - V^{\pi_w}_{h+1,\mathcal{M}(w)}\right|(s_h,a_h)\nonumber \\
& \stackrel{(i)}{\le} \widetilde{P}_{h,w}\Bar{V}_{h+1,\mathcal{M}^{(\widetilde{b}^n,\widetilde{P}^n)}(w)}^{\pi_w}(s_h,a_h) + \sum_{k=h+1}^H\frac{H\sqrt{K}}{C\sqrt{dN}} (\widetilde{\lambda}_nd+2C^2\widetilde{\zeta}_k^n(w)),
\end{align}
where $(i)$ follows from the induction hypothesis. Then we upper-bound the second term in \cref{lemma13 step1} as follows:
\begin{align}\label{lemma13 step3}
    &\hspace{-1cm} \left|\left( \widetilde{P}^n_{h,w}-P_{h,w}\right)V^{\pi_w}_
{h+1,\mathcal{M}(w)}(s_h,a_h) \right| \nonumber\\
& \stackrel{(i)}{\le} \frac{CH\sqrt{dN}}{2\sqrt{K}} \cnorm{\phi_h(s_h,a_h)}_{(\Sigma_h^n)^{-1}}^2 +\frac{H\sqrt{K}}{C\sqrt{dN}}(\widetilde{\lambda}_nd+2C^2\widetilde\zeta_h^n(w)) \nonumber \\
& \stackrel{(ii)}{\le} \widetilde{b}_h^n(s_h,a_h) +\frac{H\sqrt{K}}{C\sqrt{dN}}(\widetilde{\lambda}_nd+2C^2\widetilde\zeta_h^n(w)) ,
\end{align}
where $(i)$ follows from \Cref{lemma: difference of vfunction under different transition_2} and $(ii)$ follows from that the event $\widetilde{\mathcal{E}}_0$ occurs. Then we have:
\begin{align}
& \hspace{-1cm}\left| Q_{h,\mathcal{M}^{(r,\widetilde{P}^n)}{(w)}}^{\pi_w}(s_h,a_h) - Q_{h,\mathcal{M}(w)}^{\pi_w}(s_h,a_h) \right|\nonumber \\
& \stackrel{(i)}{\le} \min \left\{ H, \widetilde{b}_h^n (s_h,a_h)+ \widetilde{P}_{h,w}\Bar{V}_{h+1,\mathcal{M}^{(\widetilde{b}^n,\widetilde{P}^n)}(w)}^{\pi_w}(s_h,a_h) + \sum_{k=h}^H \frac{H\sqrt{K}}{C\sqrt{dN}} (\widetilde{\lambda}_nd+2C^2\widetilde{\zeta}_k^n(w)) \right\} \nonumber \\
& \stackrel{(ii)}{\le} \min \left\{ H, \widetilde{b}_h^n (s_h,a_h)+ \widetilde{P}_{h,w}\Bar{V}_{h+1,\mathcal{M}^{(\widetilde{b}^n,\widetilde{P}^n)}(w)}^{\pi_w}(s_h,a_h)\right\} + \sum_{k=h}^H \frac{H\sqrt{K}}{C\sqrt{dN}} (\widetilde{\lambda}_nd+2C^2\widetilde{\zeta}_k^n(w)) \nonumber \\
& = \Bar{Q}^{\pi}_{h,\mathcal{M}^{(\widetilde{b}^n,\widetilde{P}^n)}(w)}(s_h,a_h) + \sum_{k=h}^H \frac{H\sqrt{K}}{C\sqrt{dN}} (\widetilde{\lambda}_nd+2C^2\widetilde{\zeta}_k^n(w))
\end{align}
where $(i)$ follows from \cref{lemma13 step1,lemma13 step2,lemma13 step3} and the fact that $\left| Q_{h,\mathcal{M}^{(r,\widetilde{P}^n)}{(w)}}^{\pi_w}(s_h,a_h) - Q_{h,\mathcal{M}(w)}^{\pi_w}(s_h,a_h) \right|$ is bounded by $H$, and $(ii)$ follows from that fact that $\min \{a,b+c\} \le \min \{a,b\}+\min \{a,c\}$ if $a,b,c\ge 0$. 
Then, by the definition of $\Bar{V}^{\pi_w}_{h,\mathcal{M}^{(\widetilde{b}^n,\widetilde{P}^n)}(w)}(s_h)$, we have:
\begin{align*}
    &\hspace{-1cm}\left| V_{h,\mathcal{M}{(w)}}^{\pi_w}(s_h) - V_{h,\mathcal{M}^{(r,\widetilde{P}^n)}(w)}^{\pi_w}(s_h) \right| \\
    & = \left | \underset{\pi_w}{\mathbb{E}} \left[ Q^{\pi_w}_{h,\mathcal{M}(w)}(s_h,a_h)\right] - \underset{\pi_w}{\mathbb{E}}\left[Q^{\pi_w}_{h,\mathcal{M}^{(r,\widetilde{P}^n)}(w)}(s_h,a_h) \right]\right |\\
    & \le \underset{\pi_w}{\mathbb{E}}\left[ \left|Q^{\pi_w}_{h,\mathcal{M}(w)}(s_h,a_h)-Q^{\pi_w}_{h,\mathcal{M}^{(r,\widetilde{P}^n)}(w)}(s_h,a_h) \right|\right]\\
    & \stackrel{(i)}{\leq} \underset{\pi_w}{\mathbb{E}}\left[ \Bar{Q}^{\pi_w}_{h,\mathcal{M}^{(\widetilde{b}^n,\widetilde{P}^n)}(w)}(s_h,a_h)\right]+ \sum_{k=h}^H\frac{H\sqrt{K}}{C\sqrt{dN}} (\widetilde{\lambda}_nd+2C^2\widetilde{\zeta}_k^n(w))\\
    & = \Bar{V}^{\pi_w}_{h,\mathcal{M}^{(\widetilde{b}^n,\widetilde{P}^n)}(w)}(s_h)+ \sum_{k=h}^H\frac{H\sqrt{K}}{C\sqrt{dN}} (\widetilde{\lambda}_nd+2C^2\widetilde{\zeta}_k^n(w)),
\end{align*}
where $(i)$ follows from \cref{eq:VrP-VrPhat}. Therefore, by induction, we can conclude that:
\[
    \left| V_{\mathcal{M}{(w)}}^{\pi_w} - V_{\mathcal{M}^{(r,\widetilde{P}^n)}(w)}^{\pi_w} \right| \le \Bar{V}_{\mathcal{M}^{(\widetilde{b}^n,\widetilde{P}^n)}(w)}^{\pi_w} +\sum_{h=1}^H \frac{H \sqrt{K}}{C\sqrt{dN}} (\widetilde{\lambda}_nd+2C^2\widetilde{\zeta}_h^n(w)).
\]
Hence, taking the expectations over the context $w$ on the both sides of the above equation and applying \Cref{lemma:model2 mle} complete the proof.
\end{proof}

Next we use a similar idea to bound the estimation error of the reward function. To simplify the notation, we define \begin{equation}\label{def:wid zeta prime}
    \widetilde{\zeta}_h^{n\prime}(w) = \sum_{\tau=1}^{n-1} \underset{\substack{s_h\sim(P_{w},\pi^\tau_{w}) \\ a_h \sim \mathcal{U}(\mathcal{A})}}{\mathbb{E}}\left\|\left\langle\widetilde{\eta}^n_h(\cdot,w) - \eta_h(\cdot,w), \psi_h\left(s_h, a_h\right)\right\rangle\right\|_2^2. 
\end{equation}
\begin{lemma}\label{lemma:linear diff expect reward}
Suppose $\widetilde{\eta}_h^n$ is obtained by $\widetilde{\eta}_h^n = \mathrm{LSR}(\mathcal{G}_h^n)$ in \Cref{alg2} at step $h$ in episode $n$ and $\widetilde{f}_h^n(s,a,w)$ is the estimated reward function defined in \cref{def:wid f}. For any $(s_h,a_h,w)\in \mathcal{S}\times\mathcal{A}\times\mathcal{W}$, with probability at least $1-\delta/2$, we have:
\begin{align*}
    & \hspace{-1cm}  \left| \widetilde{f}_h^n(s_h,a_h,w) - r_h(s_h,a_h,w)\right|\\
    & \le \min \left\{\frac{C\sqrt{dN}}{2\sqrt{K}} \cnorm{\psi_h(s_h,a_h)}_{(\Lambda_h^n)^{-1}}^2,1 \right\} +\frac{\sqrt{K}}{2C\sqrt{dN}}(2\widetilde{\xi}_nd+C^2\widetilde{\zeta}_h^{n\prime}(w)).
\end{align*}
\end{lemma}
\begin{proof}
    Following from the fact that $r_h(s_h,a_h,w) \in [0,1]$ for any $(s_h,a_h,w)$ any $h\in[H],n\in[N]$ and the definition in \cref{def:wid f}, we have:
    \begin{equation}\label{lemma14 step1}
        \left| \widetilde{f}_h^n(s_h,a_h,w) - r_h(s_h,a_h,w)\right| \le \left| \left\langle\widetilde{\eta}_h^n, \psi_h(s_h,a_h,w)\right\rangle - r_h(s_h,a_h,w)\right|.
    \end{equation} 

Then for any given $(s_h,a_h,w)\in \mathcal{S}\times \mathcal{A}\times \mathcal{W}$, we obtain:
\begin{align}
& \left|\langle\widetilde{\eta}_h^n(w)-\eta_h(w), \psi_h(s_h,a_h)\rangle \right|  \nonumber\\
& \hspace{8mm} \stackrel{(i)}{\leq} \left\|\widetilde{\eta}_h^n(w)-\eta_h(w) \right\|_{\left(\Lambda_h^n\right)} \cdot \|\psi_h(s_h,a_h)\|_{\left(\Lambda_h^n\right)^{-1}} ,\label{eq:linear diff 3}
\end{align}

where $(i)$ follows from the Cauchy-Schwarz inequality. Then, we further derive that
\begin{align}\label{eq:linear diff 4}
& \hspace{-1cm}\left\|\widetilde{\eta}_h^n(w)-\eta_h(w) \right\|_{\left(\Lambda_h^n\right)}^2 \nonumber\\
&= \widetilde\xi_n \cdot\left\|\widetilde{\eta}_h^n(w)-\eta_h(w) \right\|_{\left(\Lambda_h^n\right)}^2 \nonumber\\ 
& \hspace{9mm} +\sum_{\tau=1}^{n-1} \underset{\substack{w_\tau \sim q \\ s'_h\sim(P_{w_\tau},\pi^\tau_{w_\tau}) \\ a'_h \sim \mathcal{U}(\mathcal{A})}}{\mathbb{E}}\left(\left(\widetilde{\eta}^n_h(w)-\eta_h(w)\right)^{\top} \phi_h \left(s'_h,a'_h\right) \right)^2 \nonumber\\
&\stackrel{(i)}{\leq}  2\widetilde\xi_n d + C^2  \widetilde{\zeta}_h^{n\prime}(w),
\end{align}
where $(i)$ follows from \Cref{def:model 2} and from \Cref{lemma:switch context}. Then we have:
\begin{align}\label{lemma14 step2}
& \hspace{-1cm} \left|\langle \widetilde{\eta}_h^n(w)-\eta_h(w), \psi_h(s_h,a_h)\rangle \right| \nonumber \\
& \stackrel{(i)}{\le}\|\psi_h(s_h,a_h)\|_{\left(\Lambda_h^n\right)^{-1}} \cdot \sqrt{2\widetilde\xi_n d + C^2\widetilde{\zeta}_h^{n\prime}(w)}\nonumber\\
& =  \sqrt{\frac{C\sqrt{dN}}{\sqrt{K}} \cnorm{\psi_h(s_h,a_h)}_{(\Lambda_h^n)^{-1}}^2 } \cdot \sqrt{\frac{\sqrt{K}}{C\sqrt{dN}}(2\widetilde{\xi}_nd+C^2\widetilde{\zeta}_h^{n\prime}(w))}\nonumber\\
&\stackrel{(ii)}{\le}  \frac{C\sqrt{dN}}{2\sqrt{K}} \cnorm{\psi_h(s_h,a_h)}_{(\Lambda_h^n)^{-1}}^2 +\frac{\sqrt{K}}{2C\sqrt{dN}}(2\widetilde{\xi}_nd+C^2\widetilde{\zeta}_h^{n\prime}(w)),
\end{align}

where $(i)$ follows from \cref{eq:linear diff 3,eq:linear diff 4} and $(ii)$ follows from the fact that $ab\le \frac{1}{2}a^2+\frac{1}{2}b^2$. By combining \cref{lemma14 step1,lemma14 step2}, and the fact that $\left| \widetilde{f}_h^n(s_h,a_h,w) - r_h(s_h,a_h,w)\right| \le 1$, we have:
\begin{align*}
    & \hspace{-1cm}  \left| \widetilde{f}_h^n(s_h,a_h,w) - r_h(s_h,a_h,w)\right|\\
    & \le \min \left\{\frac{C\sqrt{dN}}{2\sqrt{K}} \cnorm{\psi_h(s_h,a_h)}_{(\Lambda_h^n)^{-1}}^2 +\frac{\sqrt{K}}{2C\sqrt{dN}}(2\widetilde{\xi}_nd+C^2\widetilde{\zeta}_h^{n\prime}(w)),1 \right\}\\
    & \le \min \left\{\frac{C\sqrt{dN}}{2\sqrt{K}} \cnorm{\psi_h(s_h,a_h)}_{(\Lambda_h^n)^{-1}}^2,1 \right\} +\frac{\sqrt{K}}{2C\sqrt{dN}}(2\widetilde{\xi}_nd+C^2\widetilde{\zeta}_h^{n\prime}(w)),
\end{align*}
which completes the proof.
\end{proof}

\begin{lemma}
\label{lemma: difference of vfunction under different reward_1} Define the function $\widetilde{c}_h^n(s,a) := \min \left \{\widetilde\beta_n \cnorm{\psi_h(s,a)}^2_{\left( \widetilde{\Lambda}_h^n\right)^{-1}} ,1 \right \}$, where $\widetilde\beta_n = \frac{25C\sqrt{dN}}{2\sqrt{K}}$. Assume that the event $\widetilde{\mathcal{E}}_0$ occurs. Then for any context-dependent policy $\pi_w$, with probability at least $1-\delta/2$, we have:
\begin{equation*}
    \underset{\substack{w \sim q}}{\mathbb{E}} \left| \Bar{V}_{\mathcal{M}^{(\widetilde{f}^n,\widetilde{P}^n)}(w)}^{\pi_w} - V_{\mathcal{M}^{(r,\widetilde{P}^n)}(w)}^{\pi_w} \right| \le \underset{\substack{w \sim q}}{\mathbb{E}} \Bar{V}_{\mathcal{M}^{(\widetilde{c}^n,\widetilde{P}^n)}(w)}^{\pi_w} + \frac{H\sqrt{K}}{2C\sqrt{dN}} (2\widetilde{\xi}_nd+C^2\zeta'_n).
\end{equation*}
\end{lemma} 
\begin{proof}
Note that the values of both $\widetilde{f}_h^n(s,a,w)$ and $\widetilde{c}_h^n(s,a,w)$ are restricted to $[0,1]$ for any $(s,a,w)$ and any $h\in[H], n\in[N]$ . Following from \Cref{lem:Vbar=V}, it is equivalent to prove:
\begin{equation*}
    \underset{\substack{w \sim q}}{\mathbb{E}} \left| V_{\mathcal{M}^{(\widetilde{f}^n,\widetilde{P}^n)}(w)}^{\pi_w} - V_{\mathcal{M}^{(r,\widetilde{P}^n)}(w)}^{\pi_w} \right| \le \underset{\substack{w \sim q}}{\mathbb{E}} V_{\mathcal{M}^{(\widetilde{c}^n,\widetilde{P}^n)}(w)}^{\pi_w} + \frac{H\sqrt{K}}{2C\sqrt{dN}} (2\widetilde{\xi}_nd+C^2\zeta'_n).
\end{equation*}
Then, we have:
\begin{align}
    &\hspace{-1cm} \left|V^{\pi_w} _{\mathcal{M}^{(r,\widetilde{P}^n)}(w)} - V^{\pi_w} _{\mathcal{M}^{(\widetilde{f}^n,\widetilde{P}^n)}(w)} \right| \nonumber\\
    & \stackrel{(i)}{=} \left|\sum_{h=1}^H \underset{\substack{(s_h,a_h)\sim(\widetilde{P}^n_{w},\pi_{w})}}{\mathbb{E}}\widetilde{f}_h^n\left(s_h, a_h,w\right)-r_h\left(s_h, a_h,w\right)\right|\nonumber \\
    & \le \sum_{h=1}^H \underset{\substack{(s_h,a_h)\sim(\widetilde{P}^n_{w},\pi_{w})}}{\mathbb{E}}\left|\widetilde{f}_h^n\left(s_h, a_h,w\right)-r_h\left(s_h, a_h,w\right)\right|\nonumber \\
    &\stackrel{(ii)}{\le}  \sum_{h=1}^H \left[\underset{\substack{(s_h,a_h)\sim(\widetilde{P}^n_{w},\pi_{w})}}{\mathbb{E}} \beta_n\cnorm{\psi_h\left(s_h, a_h\right)}_{\left(\widetilde{\Lambda}_h^n\right)^{-1}}^2 + \frac{\sqrt{K}}{2C\sqrt{dN}}(2\widetilde{\xi}_nd+C^2\widetilde{\zeta}_h^{n\prime}(w)) \right] \nonumber\\
    & = V_{\mathcal{M}^{(\widetilde{c}^n,\widetilde{P}^n)}(w)}^{\pi_w} + \sum_{h=1}^H\frac{\sqrt{K}}{2C\sqrt{dN}}(2\widetilde{\xi}_nd+C^2\widetilde{\zeta}_h^{n\prime}(w)),
\end{align}
where $(i)$ follows from \Cref{simulation lemma} and $(ii)$ follows from \Cref{lemma:linear diff expect reward} and from the occurrence of event $\widetilde{\mathcal{E}}_0$. Then, taking expectations over the context $w$ on the both sides of the above equation and applying \Cref{lemma:model2 lsr} complete the proof.
\end{proof}

\begin{lemma} \label{ lemma:difference of vfunction under different transition and reward model2} Suppose the event $\widetilde{\mathcal{E}}_0$ occurs. Then for any context-dependent policy $\pi_w$, with probability at least $1-\delta/2$, we have
\begin{align*}
    &\underset{\substack{w \sim q}}{\mathbb{E}}\left|\Bar{V}^{\pi_w} _{\mathcal{M}^{(\widetilde{f}^n,\widetilde{P}^n)}(w)}  -V^{\pi_w} _{\mathcal{M}{(w)}} \right| \\
    & \hspace{15mm}\le \underset{w\sim q}{\mathbb{E}} \Bar{V}^{\pi_w} _{\mathcal{M}^{(\widetilde{b}^n,\widetilde{P}^n)}(w)} + \underset{w\sim q}{\mathbb{E}} \Bar{V}^{\pi_w} _{\mathcal{M}^{(\widetilde{c}^n,\widetilde{P}^n)}(w)} + \frac{H^2\sqrt{K}}{2C\sqrt{dN}}(2\widetilde{\xi}_nd+C^2\zeta'_n + 2\widetilde{\lambda}_nd + 4C^2\zeta_n).
\end{align*}

\end{lemma}
\begin{proof}
By combining the bounds on the estimation error of both the reward and the transition kernel, characterized respectively in \Cref{lemma: difference of vfunction under different transition_2} and \Cref{lemma: difference of vfunction under different reward_1}, and the fact that $H\ge 1$, we have:
\begin{align*}
&\hspace{-18mm}\underset{\substack{w \sim q}}{\mathbb{E}}\left|\Bar{V}^{\pi_w} _{\mathcal{M}^{(\widetilde{f}^n,\widetilde{P}^n)}(w)}  -V^{\pi_w} _{\mathcal{M}{(w)}} \right| \\
= &\underset{\substack{w \sim q}}{\mathbb{E}}\left|\Bar{V}^{\pi_w} _{\mathcal{M}^{(\widetilde{f}^n,\widetilde{P}^n)}(w)}  -V^{\pi_w} _{\mathcal{M}^{(r,\widetilde{P}^n)}(w)} 
    +V^{\pi_w} _{\mathcal{M}^{(r,\widetilde{P}^n)}(w)} -V^{\pi_w} _{\mathcal{M}{(w)}} \right|\\
\le & \underset{\substack{w \sim q}}{\mathbb{E}}\left|\Bar{V}^{\pi_w} _{\mathcal{M}^{(\widetilde{f}^n,\widetilde{P}^n)}(w)}  -V^{\pi_w} _{\mathcal{M}^{(r,\widetilde{P}^n)}(w)} \right|
    + \underset{w\sim q}{\mathbb{E}} \left|V^{\pi_w} _{\mathcal{M}^{(r,\widetilde{P}^n)}(w)} -V^{\pi_w} _{\mathcal{M}{(w)}} \right|\\
\le & \underset{\substack{w \sim q}}{\mathbb{E}}\Bar{V}^{\pi_w} _{\mathcal{M}^{(\widetilde{b}^n,\widetilde{P}^n)}(w)} + \underset{w\sim q}{\mathbb{E}} \Bar{V}^{\pi_w} _{\mathcal{M}^{(\widetilde{c}^n,\widetilde{P}^n)}(w)} +\frac{H^2\sqrt{K}}{2C\sqrt{dN}}(2\widetilde{\xi}_nd+C^2\zeta'_n + 2\widetilde{\lambda}_nd + 4C^2\zeta_n),
\end{align*}
which completes the proof.
\end{proof}

\subsection{Proof of \Cref{thm: model 2}.}
We first restate \Cref{thm: model 2} below.
\begin{theorem}[Restatement of \Cref{thm: model 2}]
Consider a CMDP  with varying linear weights as defined in \Cref{def:model 2}. Under \Cref{ass:mu(sw),ass:pmin}, for any $\delta\in(0,1)$,  with probability at least $1-3\delta/2$, the sequence of policies $\pi_{w_n}^n$ generated by \Cref{alg2} satisfies that
\begin{align*}
\frac{1}{N}\sum_{n=1}^N&\underset{w\sim q}{\mathbb{E}}\left [V^{\pi^*_{w}} _{\mathcal{M}{(w)}} - V^{\pi^n_{w}} _{\mathcal{M}{(w)}}\right]\\
\le & 912CH^2\sqrt{\frac{d^3K}{N}}\mathrm{log}\left(1+\frac{N}{\widetilde{\lambda}d}\right) \\
&+\frac{H^2}{C}\sqrt{\frac{K}{dN}}\left(2\widetilde{\xi}_Nd+C^2\mathrm{log}\left(\frac{2HN|\Psi_3|}{\delta}\right) +4\widetilde{\lambda}_Nd+8C^2\mathrm{log}\left(\frac{2HN|\Psi_2|}{\delta} \right) \right),
\end{align*}
where $C = \sqrt{\frac{p_{\mathrm{max}}}{p_{\mathrm{min}}}}$, $\widetilde{\lambda}_n = \widetilde{\gamma}_1d\mathrm{log}(2nH/\delta)$, $\widetilde{\xi}_n = \widetilde{\gamma}_1d\mathrm{log}(2nH/\delta)$, $\widetilde{\gamma}_1,\widetilde{\gamma}_2 = \mathcal{O}(1)$ and $\widetilde{\lambda} = \min\{\widetilde{\lambda}_1, \widetilde{\xi}_1\}$. 
To achieve an $\epsilon$ average sub-optimality gap, at most $\mathcal{O}\left(\frac{H^4d^3K\mathrm{log}^2(|\Psi_2||\Psi_3|/\delta^2)}{\epsilon^2} \cdot \frac{p_{\mathrm{max}}}{p_{\mathrm{min}}}\right)$ episodes are needed.
\end{theorem}

\begin{proof}
First, we derive an optimistic upper bound for the expectation of optimal value function.
\begin{align}
    \underset{w\sim q}{\mathbb{E}} V^{\pi^*_{w}} _{\mathcal{M}{(w)}} \stackrel{(i)}{\le} &  \underset{w\sim q}{\mathbb{E}}\left [\Bar{V}^{\pi^*_{w}} _{\mathcal{M}^{(\widetilde{f}^n,\widetilde{P}^n)}{(w)}} + \Bar{V}^{\pi^*_{w}} _{\mathcal{M}^{(\widetilde{b}^n,\widetilde{P}^n)}{(w)}} +\Bar{V}^{\pi^*_{w}} _{\mathcal{M}^{(\widetilde{c}^n,\widetilde{P}^n)}{(w)}}\right]\nonumber\\  & \hspace{15mm} +\frac{H^2\sqrt{K}}{2C\sqrt{dN}}(2\widetilde{\xi}_nd+C^2\zeta'_n + 2\widetilde{\lambda}_nd + 4C^2\zeta_n)\nonumber\\
    \stackrel{(ii)}{\le} &  \underset{w\sim q}{\mathbb{E}}\left [\doublebar{V}^{\pi^*_{w}} _{\mathcal{M}^{(\widetilde{f}^n+\widetilde{b}^n+\widetilde{c}^n,\widetilde{P}^n)}{(w)}} \right] + \frac{H^2\sqrt{K}}{2C\sqrt{dN}}(2\widetilde{\xi}_nd+C^2\zeta'_n + 2\widetilde{\lambda}_nd + 4C^2\zeta_n)\nonumber \\
    \stackrel{(iii)}{\le} &  \underset{w\sim q}{\mathbb{E}}\left [\doublebar{V}^{\pi^n_{w}} _{\mathcal{M}^{(\widetilde{f}^n+\widetilde{b}^n+\widetilde{c}^n,\widetilde{P}^n)}{(w)}} \right] + \frac{H^2\sqrt{K}}{2C\sqrt{dN}}(2\widetilde{\xi}_nd+C^2\zeta'_n + 2\widetilde{\lambda}_nd + 4C^2\zeta_n) \nonumber \\
    \stackrel{(iv)}{\le} &  \underset{w\sim q}{\mathbb{E}}\left [\Bar{V}^{\pi^n_{w}} _{\mathcal{M}^{(\widetilde{f}^n,\widetilde{P}^n)}{(w)}} + \Bar{V}^{\pi^n_{w}} _{\mathcal{M}^{(\widetilde{b}^n,\widetilde{P}^n)}{(w)}} +\Bar{V}^{\pi^n_{w}} _{\mathcal{M}^{(\widetilde{c}^n,\widetilde{P}^n)}{(w)}} \right] \nonumber\\ &\hspace{15mm}+ \frac{H^2\sqrt{K}}{2C\sqrt{dN}}(2\widetilde{\xi}_nd+C^2\zeta'_n + 2\widetilde{\lambda}_nd + 4C^2\zeta_n),\label{eq:value decomp}
\end{align}
where $(i)$ follows from \Cref{ lemma:difference of vfunction under different transition and reward model2}, $(ii)$ follows from \Cref{Vbar<=Vdoublebar}, $(iii)$ follows from the greedy policy $\pi^n_{w} = \mathrm{argmax}_\pi \doublebar{V}^\pi_{\mathcal{M}^{(\widetilde{f}^n+\widetilde{b}^n+\widetilde{c}^n,\widetilde{P}^n)}(w)}$ and $(iv)$ follows from \Cref{lem:Vdoublebar<=Vdoublebar}.
Then the average suboptimality gap can be bounded as 
\begin{align}
&\hspace{-8mm}\frac{1}{N}\sum_{n=1}^N \underset{w\sim q}{\mathbb{E}}\left [V^{\pi^*_{w}} _{\mathcal{M}{(w)}} - V^{\pi^n_{w}} _{\mathcal{M}{(w)}}\right]\nonumber\\
&\stackrel{(i)}{\le}  \frac{1}{N}\sum_{n=1}^N \underset{w\sim q}{\mathbb{E}}\left [\Bar{V}^{\pi^n_{w}} _{\mathcal{M}^{(\widetilde{f}^n,\widetilde{P}^n)}{(w)}} + \Bar{V}^{\pi^n_{w}} _{\mathcal{M}^{(\widetilde{b}^n,\widetilde{P}^n)}{(w)}} + \Bar{V}^{\pi^n_{w}} _{\mathcal{M}^{(\widetilde{c}^n,\widetilde{P}^n)}{(w)}} - V^{\pi^n_{w}} _{\mathcal{M}{(w)}}\right] \nonumber \\
& \hspace{15mm}+ \frac{H^2\sqrt{K}}{2C\sqrt{dN}}(2\widetilde{\xi}_nd+C^2\zeta'_n + 2\widetilde{\lambda}_nd + 4C^2\zeta_n)\nonumber\\
&\stackrel{(ii)}{\le}  \frac{1}{N}\sum_{n=1}^N \underset{w\sim q}{\mathbb{E}}\left [2\Bar{V}^{\pi^n_{w}} _{\mathcal{M}^{(\widetilde{b}^n,\widetilde{P}^n)}{(w)}}+2\Bar{V}^{\pi^n_{w}} _{\mathcal{M}^{(\widetilde{c}^n,\widetilde{P}^n)}{(w)}}\right] + \frac{H^2\sqrt{K}}{C\sqrt{dN}}(2\widetilde{\xi}_nd+C^2\zeta'_n + 2\widetilde{\lambda}_nd + 4C^2\zeta_n), \label{Eq: thm4-decompose}
\end{align}
where $(i)$ follows from \cref{eq:value decomp}, and $(ii)$ follows from \Cref{ lemma:difference of vfunction under different transition and reward model2}.

We next provide an upper bound on $\sum_{n=1}^N\underset{w\sim q}{\mathbb{E}} \Bar{V}^{\pi^n_{w}} _{\mathcal{M}^{(\widetilde{b}^n,\widetilde{P}^n)}{(w)}}$. Define $g_h^n(s,a,w) = (\widetilde{P}_{h,w}^n - P_{h,w}) \doublebar{V}^{\pi^n_w}_{h+1,\mathcal{M}^{(\widetilde{b}^n,\widetilde{P}^n)}(w)}(s_h,a_h)$. Following from \Cref{lemma: value trans1} in \Cref{app: AUXILIARY LEMMAS}, we have:
\begin{equation}\label{eq: value trans1_2}
\sum_{n=1}^N\underset{w\sim q}{\mathbb{E}}\Bar{V}^{\pi^n_w}_{\mathcal{M}^{(\widetilde{b}^n,\widetilde{P}^n)}(w)}  \le \sum_{n=1}^N\underset{w\sim q}{\mathbb{E}}V^{\pi^n_w}_{\mathcal{M}^{(\widetilde{b}^n,P)}(w)}+  \sum_{n=1}^N\underset{w\sim q}{\mathbb{E}}V^{\pi^n_w}_{\mathcal{M}^{(g^n,P)}(w)}.
\end{equation}

For the first term in the right-hand-side of \cref{eq: value trans1_2}, we obtain an upper bound on the summation of the expected value functions $V^{\pi_w}_{\mathcal{M}^{(\widetilde{b}^n,P)}(w)}$ as follows:
\begin{align}
\sum_{n=1}^N\underset{w\sim q}{\mathbb{E}} V^{\pi^n_{w}} _{\mathcal{M}^{(\widetilde{b}^n,P)}(w)} 
= & \sum_{n=1}^N \sum _{h=1}^H  
\underset{\substack{w \sim q\\(s_h, a_h) \sim (P_{w},\pi ^n_{w})}}{\mathbb{E}} \left[ \widetilde\alpha_n \cnorm{\phi_h (s_h,a_h)}^2_{(\widetilde{\Sigma}_h^n)^{-1}} \right] \nonumber\\
\stackrel{(i)}{\le} & 9K\widetilde\alpha_N \sum_{h=1}^H \sum _{n =1}^N \underset{\substack{w \sim q \\ s_h \sim (P_{w},\pi ^n_{w})\\ a_h \sim \mathcal{U}(\mathcal{A})}}{\mathbb{E}} \left[ \cnorm{\phi_h (s_h,a_h)}_{(\Sigma_h^n)^{-1}}^2\right]\nonumber\\
\stackrel{(ii)}{\le} & 18dHK \widetilde{\alpha}_N \cdot  \mathrm{log}\left(1+\frac{N}{d\widetilde\lambda} \right),\label{thm value trans2_2}
\end{align}
where $(i)$ follows because the event $\widetilde{\mathcal{E}}_0$ occurs and from the importance sampling, and $(ii)$ follows from \Cref{potential lemma}.

For the second term in the right-hand-side of \cref{eq: value trans1_2}, we derive 
\begin{align}
\sum_{n=1}^N\underset{w\sim q}{\mathbb{E}} &V^{\pi^n_{w}} _{\mathcal{M}^{(g^n,P)}(w)} \nonumber \\
\le & \sum_{n=1}^N \sum _{h=1}^H  
\underset{\substack{w \sim q\\(s_h, a_h) \sim (P_{w},\pi ^n_{w})}}{\mathbb{E}} \big |g_h^n(s_h,a_h,w) \big|\nonumber\\
\stackrel{(i)}{\le} & \sum_{n=1}^N \sum _{h=1}^H  
\frac{3\widetilde{\alpha}_n}{25} \underset{\substack{w \sim q \\ (s_h,a_h)\sim(P_{w},\pi_{w})}}{\mathbb{E}} \cnorm{\phi_h(s_h,a_h)}_{(\Sigma_h^n)^{-1}}^2 +\frac{H^2\sqrt{KN}}{C\sqrt{d}}(\widetilde{\lambda}_Nd+2C^2\zeta_N) \nonumber\\
\stackrel{(ii)}{\le} & \sum_{n=1}^N \sum _{h=1}^H  
\frac{3K\widetilde{\alpha}_n}{25} \underset{\substack{w \sim q \\ s_h\sim(P_{w},\pi_{w}) \\ a_h \sim \mathcal{U}(\mathcal{A})}}{\mathbb{E}} \cnorm{\phi_h(s_h,a_h)}_{(\Sigma_h^n)^{-1}}^2 +\frac{H^2\sqrt{KN}}{C\sqrt{d}}(\widetilde{\lambda}_Nd+2C^2\zeta_N) \nonumber\\
\stackrel{(iii)}{\le} & \frac{3HK\widetilde{\alpha}_N}{25} \cdot 2d \mathrm{log}\left(1+\frac{N}{d\widetilde\lambda} \right)+\frac{H^2\sqrt{KN}}{C\sqrt{d}}(\widetilde{\lambda}_Nd+2C^2\zeta_N),\label{thm value trans3_2}
\end{align}
where $(i)$ follows from \Cref{lemma:linear diff expect} and the fact that $\doublebar{V}^{\pi_w}_{h,\mathcal{M}^{(\widetilde{b}^n,\widetilde{P}^n)}(w)}(s_h,a_h) \le 3H$ for any $h\in[H]$, $(ii)$ follows from the importance sampling, and $(iii)$ follows from \Cref{potential lemma}. Then by combining \cref{eq: value trans1_2,thm value trans2_2,thm value trans3_2}, we have:
\begin{equation}\label{eq:value bound1}
    \sum_{n=1}^N \underset{w\sim q}{\mathbb{E}} \Bar{V}^{\pi_w}_{\mathcal{M}^{(\widetilde{b}^n,\widetilde{P}^n)}(w)} \le \frac{456dHK\widetilde{\alpha}_n}{25} \cdot \mathrm{log}\left(1+\frac{N}{d\widetilde\lambda} \right)+\frac{H^2\sqrt{KN}}{C\sqrt{d}}(\widetilde{\lambda}_Nd+2C^2\zeta_N).
\end{equation}
Next we provide an upper bound on $\sum_{n=1}^N\underset{w\sim q}{\mathbb{E}} \Bar{V}^{\pi^n_{w}} _{\mathcal{M}^{(\widetilde{c}^n,\widetilde{P}^n)}{(w)}}$. Define $l_h^n(s,a,w) = (\widetilde{P}_{h,w}^n - P_{h,w}) \Bar{V}^{\pi^n_w}_{h+1,\mathcal{M}^{(\widetilde{c}^n,\widetilde{P}^n)}(w)}(s_h,a_h)$. Following from \Cref{lemma: value trans1} in \Cref{app: AUXILIARY LEMMAS}, we have:
\begin{equation}\label{eq: value reward1_2}
\sum_{n=1}^N\underset{w\sim q}{\mathbb{E}}\Bar{V}^{\pi^n_w}_{\mathcal{M}^{(\widetilde{c}^n,\widetilde{P}^n)}(w)}  \le \sum_{n=1}^N\underset{w\sim q}{\mathbb{E}}V^{\pi^n_w}_{\mathcal{M}^{(\widetilde{c}^n,P)}(w)}+  \sum_{n=1}^N\underset{w\sim q}{\mathbb{E}}V^{\pi^n_w}_{\mathcal{M}^{(l^n,P)}(w)}.
\end{equation}
For the first term in \cref{eq: value reward1_2}, we obtain an upper bound on the summation of the expected value function $V^{\pi^n_w}_{\mathcal{M}^{(\widetilde{c}^n,P)}(w)}$ as follows:
\begin{align}
\sum_{n=1}^N\underset{w\sim q}{\mathbb{E}} V^{\pi^n_{w}} _{\mathcal{M}^{(\widetilde{c}^n,P)}(w)} 
= & \sum_{n=1}^N \sum _{h=1}^H  
\underset{\substack{w \sim q\\(s_h, a_h) \sim (P_{w},\pi ^n_{w})}}{\mathbb{E}} \left[ \widetilde\beta_n \cnorm{\psi_h (s_h,a_h)}^2_{(\widetilde{\Lambda}_h^n)^{-1}} \right] \nonumber\\
\stackrel{(i)}{\le} & 9K\widetilde\beta_N \sum_{h=1}^H \sum _{n =1}^N \underset{\substack{w \sim q \\ s_h \sim (P_{w},\pi ^n_{w})\\ a_h \sim \mathcal{U}(\mathcal{A})}}{\mathbb{E}} \left[ \cnorm{\psi_h (s_h,a_h)}_{(\Lambda_h^n)^{-1}}^2\right]\nonumber\\
\stackrel{(ii)}{\le} & 18dHK \widetilde{\beta}_N \cdot  \mathrm{log}\left(1+\frac{N}{d\widetilde\lambda} \right),\label{thm value reward2_2}
\end{align}
where $(i)$ follows because the event $\widetilde{\mathcal{E}}_0$ occurs and from the importance sampling, and $(ii)$ follows from \Cref{potential lemma}. Then, since $\doublebar{V}^{\pi_w}_{h,\mathcal{M}^{(\widetilde{c}^n,\widetilde{P}^n)}(w)}(s,a) \le 3H$ for any $h\in[H]$, we bound the second term in the right-hand-side of \cref{eq: value reward1_2} similarly to \cref{thm value trans3_2} and obtain:
\begin{equation}\label{thm value reward3_2}
    \sum_{n=1}^N\underset{w\sim q}{\mathbb{E}}V^{\pi^n_w}_{\mathcal{M}^{(l^n,P)}(w)} \le \frac{3HK\widetilde{\alpha}_N}{25} \cdot 2d \mathrm{log}\left(1+\frac{N}{d\widetilde\lambda} \right)+\frac{H^2\sqrt{KN}}{C\sqrt{d}}(\widetilde{\lambda}_Nd+2C^2\zeta_N).
\end{equation}
Then by combining \cref{eq: value reward1_2,thm value reward2_2,thm value reward3_2}, we obtain:
\begin{equation}\label{eq:value bound2}
    \sum_{n=1}^N\underset{w\sim q}{\mathbb{E}}\Bar{V}^{\pi^n_w}_{\mathcal{M}^{(\widetilde{c}^n,\widetilde{P}^n)}(w)} \le \left(\frac{6HK\widetilde{\alpha}_N}{25} +18HK\widetilde{\beta}_N\right)\cdot d \mathrm{log}\left(1+\frac{N}{d\widetilde\lambda} \right) +\frac{H^2\sqrt{KN}}{C\sqrt{d}}(\widetilde{\lambda}_Nd+2C^2\zeta_N).
\end{equation}
By substituting \cref{eq:value bound1,eq:value bound2} into \cref{Eq: thm4-decompose}, we have
\[
\begin{aligned}
&\hspace{-8mm}\frac{1}{N}\sum_{n=1}^N \underset{w\sim q}{\mathbb{E}}\left [V^{\pi^*_{w}} _{\mathcal{M}{(w)}} - V^{\pi^n_{w}} _{\mathcal{M}{(w)}}\right]\\
&{\le} \left(\frac{912dHK\widetilde{\alpha}_N}{25N} + 18HK\widetilde{\beta}_N \right)\cdot  \mathrm{log}\left(1+\frac{N}{d\widetilde\lambda} \right)+\frac{H^2\sqrt{K}}{C\sqrt{dN}}(2\widetilde{\xi}_Nd+C^2\zeta'_N + 4\widetilde{\lambda}_Nd + 8C^2\zeta_N)\\
&\le 912CH^2\sqrt{\frac{d^3K}{N}}\mathrm{log}\left(1+\frac{N}{\widetilde{\lambda}d}\right) \\
&\hspace{15mm}+\frac{H^2}{C}\sqrt{\frac{K}{dN}}\left(2\widetilde{\xi}_Nd+C^2\mathrm{log}\left(\frac{2HN|\Psi_3|}{\delta}\right) +4\widetilde{\lambda}_Nd+8C^2\mathrm{log}\left(\frac{2HN|\Psi_2|}{\delta} \right) \right),
\end{aligned}
\]
where $C = \sqrt{\frac{p_{\mathrm{max}}}{p_{\mathrm{min}}}}$, $\widetilde{\lambda}_n = \widetilde{\gamma}_1d\mathrm{log}(2nH/\delta)$, $\widetilde{\xi}_n = \widetilde{\gamma}_1d\mathrm{log}(2nH/\delta)$, $\widetilde{\gamma}_1,\widetilde{\gamma}_2 = \mathcal{O}(1)$ and $\widetilde{\lambda} = \min\{\widetilde{\lambda}_1, \widetilde{\xi}_1\}$.

Note that since $\widetilde{\lambda}_N = \widetilde{\mathcal{O}}(d)$, the upper bound on the average suboptimality gap is of the order  $\mathcal{O}\left(\sqrt{\frac{H^4d^3KC^2\log^2(|\Psi_2||\Psi_3|/\delta^2)}{N}}\right)$. Then it can be seen that to achieve an $\epsilon$ average sub-optimality gap, at most $\mathcal{O}\left(\frac{H^4d^3K\mathrm{log}^2(|\Psi_2||\Psi_3|/\delta^2)}{\epsilon^2} \cdot \frac{p_{\mathrm{max}}}{p_{\mathrm{min}}}\right)$ episodes are needed. This completes the proof.
\end{proof}
\section{Least Square Regression (LSR) Guarantee} \label{app: LSR guarantee}
In this section, we derive an upper bound on a LSR estimation of a generic deterministic model. To simplify the notation, we denote the instance space as $\mathcal{X}$. Our goal is to estimate a deterministic function $f^*(x)$ that belongs to a function class $\mathcal{F}:\mathcal{X} \rightarrow\mathbb{R}$. Our estimation is based on a dataset $D:=\{x_i,y_i\}_{i=1}^n$, where $x_i\sim \mathcal{D}_i=\mathcal{D}_i(x_{1:i-1},y_{1:i-1})$ and $y_i = f^*(x_i)$, where $\mathcal{D}_i$ depends on the previous samples. We further define a tangent sequence $\mathcal{D}':=\{x'_i,y'_i\}_{i=1}^n$, where $x'_i\sim \mathcal{D}_i=\mathcal{D}_i(x_{1:i-1},y_{1:i-1})$ and $y'_i = f^*(x_i')$. We obtain the estimator via the following minimization problem:
\begin{equation*}
    \hat{f} = \underset{{f\in \mathcal{F}}}{\mathrm{argmin}} \sum_{i=1}^n \|f(x_i) - f^*(x_i)\|_2^2.
\end{equation*}
We first prove the following decoupling inequality, which is inspired by Lemma~24 of \citep{agarwal2020small_pmin}.  
\begin{lemma}\label{lem:lsr step1}
    Suppose $D$ is a dataset of n samples and $D'$ is a tangent sequence. Let $L(f,D) = \sum_{i=1}^n l(f,(x_i,y_i))$ be any function that decomposes additively where $l$ is any function. We denote $\hat{f}(D)$ as an estimator taking input random variable $D$. Then:
\begin{equation*}
\mathbb{E}_D\left[\exp \left( \mathbb{E}_{D'} \left(L\left(\hat{f}(D), D'\right)\right) - L(\hat{f}(D), D) -\log |\mathcal{F}|\right)\right] \leq 1.
\end{equation*}
\end{lemma}                         
\begin{proof}
Let $\pi$ be the uniform distribution over $\mathcal{F}$ and let $g: \mathcal{F} \rightarrow \mathbb{R}$ be any function. Define $\mu(f):=$ $\frac{\exp (g(f))}{\sum_f \exp (g(f))}$, which is clearly a probability distribution. Now consider any other probability distribution $\hat{\pi}$ over $\mathcal{F}$, and we have
$$
\begin{aligned}
0 & \leq \mathrm{KL}(\hat{\pi} \| \mu)=\sum_f \hat{\pi}(f) \log (\hat{\pi}(f))+\sum_f \hat{\pi}(f) \log \left(\sum_{f'} \exp \left(g\left(f'\right)\right)\right)-\sum_f \hat{\pi}(f) g(f) \\
& =\mathrm{KL}(\hat{\pi}|| \pi)-\sum_f \hat{\pi}(f) g(f)+\log \mathbb{E}_{f \sim \pi} \exp (g(f)) \\
& \leq \log |\mathcal{F}|-\sum_f \hat{\pi}(f) g(f)+\log \mathbb{E}_{f \sim \pi} \exp (g(f)).
\end{aligned}
$$
Re-arranging the above equation, we obtain that
$$
\sum_f \hat{\pi}(f) g(f)-\log |\mathcal{F}| \leq \log \mathbb{E}_{f \sim \pi} \exp (g(f)).
$$
Now we let $\hat{\pi}(f) = 1\{\hat{f}(D) \}$ and $g(f) = \mathbb{E}_{D'}L(f,D') - L(f,D)$, and have:
\begin{align*}
    \mathbb{E}_{D'}L(f(D),D') - L(f(D),D) - \log|\mathcal{F}| \le & \log \mathbb{E}_{f \sim \pi} \frac{\exp \mathbb{E}_{\mathcal{D}'} \left(L\left(\widehat{f}(\mathcal{D}), \mathcal{D}'\right)\right)}{\exp (L(\widehat{f}(\mathcal{D}), \mathcal{D}))}\\
    \le & \log \mathbb{E}_{f \sim \pi} \frac{\mathbb{E}_{\mathcal{D}'} \exp \left(L\left(\widehat{f}(\mathcal{D}), \mathcal{D}'\right)\right)}{\exp (L(\widehat{f}(\mathcal{D}), \mathcal{D}))}.
\end{align*}
We exponentiate both sides of the above equation and then take expectation over $D$ on both sides, and obtain
\begin{align*}
\mathbb{E}_D\left[\exp(\mathbb{E}_{D'}L(f(D),D') - L(f(D),D) - \log|\mathcal{F}|) \right] \le  \mathbb{E}_D\left[ \mathbb{E}_{f \sim \pi} \frac{\mathbb{E}_{\mathcal{D}'} \exp \left(L\left(\widehat{f}(\mathcal{D}), \mathcal{D}'\right)\right)}{\exp (L(\widehat{f}(\mathcal{D}), \mathcal{D}))}\right].
\end{align*}
Note that, conditioned on $\mathcal{D}$, the samples in the tangent sequence $\mathcal{D}'$ are independent, which yields
$$
\mathbb{E}_{\mathcal{D}'} \exp \left[L\left(\widehat{f}(\mathcal{D}), \mathcal{D}'\right) \mid \mathcal{D}\right]=\prod_{i=1}^n \exp \left(\mathbb{E}_{\left(x_i, y_i\right) \sim \mathcal{D}_i}\left[l\left(f,\left(x_i, y_i\right)\right)\right]\right).
$$
Then we can conclude that 
\begin{equation*}
\mathbb{E}_D\left[\exp \left( \mathbb{E}_{D'} \left(L\left(\hat{f}(D), D'\right)\right) - L(\hat{f}(D), D) -\log |\mathcal{F}|\right)\right] \leq 1.
\end{equation*}
\end{proof}                                                  
Now we present the LSR guarantee as follows.
\begin{lemma}
Assume $|\mathcal{F}|\le \infty$ and $f^* \in \mathcal{F}$. Then with probability at least $1-\delta$, we have:
\begin{equation*}
    \sum_{i=1}^n \mathbb{E}_{x_i \sim \mathcal{D}_i}\left\|f^*(x_i)-f(x_i)\right\|_2^2 \le \log |\mathcal{F}| / \delta.
\end{equation*}
\end{lemma}
\begin{proof}
We first apply the Chernoff bound to \Cref{lem:lsr step1}. With probability $1-\delta$, we have:
\begin{equation*}
    \mathbb{E}_{D'}L(\hat{f}(D),D') \le L(\hat{f}(D),D) + \log\frac{|\mathcal{F}|}{\delta}.
\end{equation*}
Since $f^*\in \mathcal{F}$, we have $L(\hat{f}(D),D)\le L(f^*,D) = 0$. Then we derive 
\begin{align*}
    \mathbb{E}_{D'}L(\hat{f}(D),D') & = \mathbb{E}_{D'}\left[ \sum_{i=1}^n \|\hat{f}(x_i')-f^*(x_i')\|^2 \Bigg|D\right]\\
    & = \underset{x_i'\sim D_i}{\mathbb{E}}\left[ \sum_{i=1}^n \|\hat{f}(x_i')-f^*(x_i')\|^2 \right]\\
    & = \sum_{i=1}^n  \underset{x_i\sim D_i}{\mathbb{E}} \|\hat{f}(x_i)-f^*(x_i)\|^2,
\end{align*}
which completes the proof.
\end{proof}
\section{Auxiliary Lemmas}\label{app: AUXILIARY LEMMAS}
The following lemma proves that the truncated value function is equal to the value function if the reward function is bounded by one.
\begin{lemma}\label{lem:Vbar=V}
    For a generic reward function $r'$ such that $r'_h(s,a,w)\in[0,1]$ for any $(s,a,w)$ and any $h\in[H]$, a generic transition kernel $P'$, and any context $w$, we have:
    \begin{equation*}
        \Bar{V}^{\pi_w}_{\mathcal{M}^{(r',P')}(w)} = V^{\pi_w}_{\mathcal{M}^{(r',P')}(w)},
    \end{equation*}
    where $\Bar{V}^{\pi_w}_{\mathcal{M}^{(r',P')}(w)}$ is defined by \cref{def:bar}.
\end{lemma}
\begin{proof}
    We develop the proof by induction. For the case $h=H+1$, we have $\Bar{V}^{\pi_w}_{H+1,\mathcal{M}^{(r',P')}(w)}(s_{H+1}) = 0=V^{\pi_w}_{H+1,\mathcal{M}^{(r',P')}(w)}(s_{H+1})$. Assume that $\Bar{V}^{\pi_w}_{h+1,\mathcal{M}^{(r',P')}(w)}(s_{h+1}) = V^{\pi_w}_{h+1,\mathcal{M}^{(r',P')}(w)}(s_{h+1})$ holds for any $s_{h+1}$. Then we have:
    \begin{align*}
        \Bar{Q}^{\pi_w}_{h,\mathcal{M}^{(r',P')}(w)}(s_{h},a_h) &= \min \left\{ H, r'_h(s_h,a_h,w)+ P'_{h,w}\Bar{V}^{\pi}_{h+1,\mathcal{M}^{(r',P')}(w)}(s_h,a_h)\right\}\\
        & \stackrel{(i)}{=} \min \left\{ H, r'_h(s_h,a_h,w)+ P'_{h,w}{V}^{\pi}_{h+1,\mathcal{M}^{(r',P')}(w)}(s_h,a_h)\right\}\\
        & = \min \left\{ H, Q^{\pi_w}_{h,\mathcal{M}^{(r',P')}(w)}(s_{h},a_h)\right\}\\
        & \stackrel{(ii)}{=} Q^{\pi_w}_{h,\mathcal{M}^{(r',P')}(w)}(s_{h},a_h),
    \end{align*}
    where $(i)$ follows from the induction hypothesis and $(ii)$ follows from the fact that the reward function $r'$ is always bounded by $1$. By taking expectations on the both sides of the above equation, we conclude that for all $s_h$,    $$\Bar{V}^{\pi_w}_{h,\mathcal{M}^{(r',P')}(w)}(s_{h}) = V^{\pi_w}_{h,\mathcal{M}^{(r',P')}(w)}(s_{h}),$$which completes the proof.
\end{proof}

Next, we present two lemmas to prove the relationship between the value function of a sum of reward functions and the sum of value functions with each corresponding to a reward function.
\begin{lemma}
    \label{Vbar<=Vdoublebar}
    Consider three MDPs denoted as $(\mathcal{S},\mathcal{A},P',r^{(i)},H)$ for $i=1,2,3$, where $P'$ is a generic transition kernel and $r^{(i)}$ are generic reward functions. Then for any context $w$ and context-dependent policy $\pi_w$, we have:
    \begin{equation*}
        \sum_{i=1}^3\Bar{V}^{\pi_w}_{\mathcal{M}^{(r^{(i)},P')}(w)} \le \doublebar{V}^{\pi_w}_{\mathcal{M}^{(r^{(1)}+r^{(2)}+r^{(3)},P'')}(w)}.
    \end{equation*}
\end{lemma}
\begin{proof}
    We develop the proof by induction. For the case $h=H+1$, we have $$ \sum_{i=1}^3\Bar{V}^{\pi_w}_{H+1,\mathcal{M}^{(r^{(i)},P')}(w)}(s_{H+1}) = 0 = \doublebar{V}^{\pi_w}_{H+1,\mathcal{M}^{(r^{(1)}+r^{(2)}+r^{(3)},P')}(w)}(s_{H+1}).$$ We assume that $ \sum_{i=1}^3\Bar{V}^{\pi_w}_{H+1,\mathcal{M}^{(r^{(i)},P')}(w)}(s_{h+1}) \le \doublebar{V}^{\pi_w}_{H+1,\mathcal{M}^{(r^{(1)}+r^{(2)}+r^{(3)},P')}(w)}(s_{H+1})$ holds for any $s_{h+1}$. Then by the definition in \cref{def:bar}, we have:
    \begin{align*}
        & \hspace{-1cm}\sum_{i=1}^3\Bar{Q}^{\pi_w}_{h,\mathcal{M}^{(r^{(i)},P')}(w)}(s_{h},a_h) \nonumber \\
        &\le \min \left\{ 3H, \sum_{i=1}^3 \left[r^{(i)}_h(s_h,a_h,w)+ P'_{h,w}\Bar{V}^{\pi_w}_{h+1,\mathcal{M}^{(r^{(i)},P')}(w)}(s_h,a_h)\right]\right\}\nonumber \\
        & \stackrel{(i)}{\le} \min \left\{ 3H, \sum_{i=1}^3 \left[r^{(i)}_h(s_h,a_h,w)\right]+ P'_{h,w}\doublebar{V}^{\pi_w}_{h+1,\mathcal{M}^{(r^{(1)}+r^{(2)}+r^{(3)},P')}(w)}(s_h,a_h)\right\}\nonumber \\
        & = \doublebar{Q}^{\pi_w}_{h,\mathcal{M}^{(r^{(1)}+r^{(2)}+r^{(3)},P')}(w)}(s_h,a_h),
    \end{align*}
    where $(i)$ follows from the induction hypothesis. By taking expectations on the both sides of the above equation, we conclude that for any $s_h$, 
    \begin{equation*}
\sum_{i=1}^3\Bar{V}^{\pi_w}_{h,\mathcal{M}^{(r^{(i)},P')}(w)}(s_{h}) \le \doublebar{V}^{\pi_w}_{h,\mathcal{M}^{(r^{(1)}+r^{(2)}+r^{(3)},P')}(w)}(s_{h}),
    \end{equation*}
    which completes the proof.
\end{proof}
\begin{lemma}
    \label{lem:Vdoublebar<=Vdoublebar}
    Suppose there are three MDPs denoted as: $(\mathcal{S},\mathcal{A},P',r^{(i)},H)$ for $i=1,2,3$ where $P'$ is a generic transition kernel and $r^{(i)}$ are generic reward functions. Then for any context $w$ and context-dependent policy $\pi_w$, we have:
    \begin{equation*}
        \doublebar{V}^{\pi_w}_{\mathcal{M}^{(r^{(1)}+r^{(2)}+r^{(3)},P')}(w)} \le \sum_{i=1}^3\doublebar{V}^{\pi_w}_{\mathcal{M}^{(r^{(i)},P')}(w)}.
    \end{equation*}
\end{lemma}
\begin{proof}
    We develop the proof by induction. For the case $h=H+1$, we have $\doublebar{V}^{\pi_w}_{H+1,\mathcal{M}^{(r^{(1)}+r^{(2)}+r^{(3)},P')}(w)} (s_{H+1})=0=\sum_{i=1}^3\doublebar{V}^{\pi_w}_{H+1,\mathcal{M}^{(r^{(i)},P')}(w)}(s_{H+1})$.\\
    We assume that $$\doublebar{V}^{\pi_w}_{h+1,\mathcal{M}^{(r^{(1)}+r^{(2)}+r^{(3)},P')}(w)} (s_{h+1}) \le \sum_{i=1}^3\doublebar{V}^{\pi_w}_{h+1,\mathcal{M}^{(r^{(i)},P')}(w)}(s_{h+1})$$ holds for any $s_{h+1}$. Then by the definition in \cref{def:doublebar}, we have:
    \begin{align*}
    & \hspace{-1cm}\doublebar{Q}^{\pi_w}_{h,\mathcal{M}^{(r^{(1)}+r^{(2)}+r^{(3)},P')}(w)}(s_h,a_h) \\
    &= \min\left\{ 3H,\sum_{i=1}^3 r^{(i)}_h(s_h,a_h,w)+P'_{h,w}\doublebar{V}^{\pi_w}_{h+1,\mathcal{M}^{(r^{(1)}+r^{(2)}+r^{(3)},P')}(w)}(s_h,a_h)\right\}\\
    & \stackrel{(i)}{\le} \min\left\{ 3H,\sum_{i=1}^3 \left[r^{(i)}_h(s_h,a_h,w)\right] + \sum_{i=1}^3 \left[P'_{h,w}\doublebar{V}^{\pi_w}_{h+1,\mathcal{M}^{(r^{(i)},P')}(w)}(s_h,a_h)\right] \right\}\\
    & \stackrel{(ii)}{\le} \sum_{i=1}^3\min \left\{ 3H,r^{(i)}_h(s_h,a_h,w) + P'_{h,w}\doublebar{V}^{\pi_w}_{h+1,\mathcal{M}^{(r^{(i)},P')}(w)}(s_h,a_h)\right\}\\
    & = \sum_{i=1}^3\doublebar{Q}^{\pi_w}_{h,\mathcal{M}^{(r^{(i)},P')}(w)}(s_{h},a_h),
    \end{align*}
    where $(i)$ follows from the induction hypothesis and $(ii)$ follows from the fact that $\min\{a,b+c\} \le \min\{a,b\} + \min\{a,c\}$ if $a,b,c\ge0$. By taking expectations on the both sides of the above equation, we conclude that for any $s_h$, we have:
    $$\doublebar{V}^{\pi_w}_{h,\mathcal{M}^{(r^{(1)}+r^{(2)}+r^{(3)},P')}(w)} (s_{h}) \le \sum_{i=1}^3\doublebar{V}^{\pi_w}_{h,\mathcal{M}^{(r^{(i)},P')}(w)}(s_{h}).$$ By induction, we complete the proof.
\end{proof}
We next present a lemma to upper-bound the difference between a truncated value function and a value function with the same generic reward function but different transition kernels by a constructed bounded value function.
The detailed lemma is presented as follows:
\begin{lemma}\label{lemma: value trans1}
    For a generic reward function $r'$, any two transition kernels $P'$ and $P''$ and any context $w$, we can obtain:
    \[
    \doublebar{V}^{\pi_w}_{\mathcal{M}^{(r',P')}(w)} - V^{\pi_w}_{\mathcal{M}^{(r',P'')}(w)} \le V^{\pi_w}_{\mathcal{M}^{(g^n,P'')}(w)},
    \]
    where $g_h^n(s,a,w) := (P'_{h,w} - P''_{h,w}) \doublebar{V}^{\pi_w}_{h+1,\mathcal{M}^{(r',P')}}(s,a)$.
\end{lemma}
\begin{proof}
For any context $w$, we have:
\begin{align*}
    & \doublebar{V}^{\pi_w}_{\mathcal{M}^{(r',P')}(w)} - V^{\pi_w}_{\mathcal{M}^{(r',P'')}(w)} \nonumber\\
    & \stackrel{(i)}{\le}  \underset{\pi_w}{\mathbb{E}}\left [ P'_{1,w} \doublebar{V}^{\pi_w}_{2,\mathcal{M}^{(r',P')}(w)}(s_1,a_1) -  P''_{1,w} V^{\pi_w}_{2,\mathcal{M}^{(r',P'')}(w)}(s_1,a_1)\right]\nonumber\\
    &=  \underset{\pi_w}{\mathbb{E}}\left [ \left (P'_{1,w} - P''_{1,w}\right) \doublebar{V}^{\pi_w}_{2,\mathcal{M}^{(r',P')}(w)}(s_1,a_1) + P''_{1,w} \left( \doublebar{V}^{\pi_w}_{2,\mathcal{M}^{(r',P')}(w)} - V^{\pi_w}_{2,\mathcal{M}^{(r',P'')}(w)}\right)(s_1,a_1)\right]\nonumber\\
    &=  \underset{\pi_w}{\mathbb{E}}\left [ g_1^n(s_1,a_1,w) + P''_{1,w} \left( \doublebar{V}^{\pi_w}_{2,\mathcal{M}^{(r',P')}(w)} - V^{\pi_w}_{2,\mathcal{M}^{(r',P'')}(w)}\right)(s_1,a_1)\right]\nonumber\\
    & \stackrel{(ii)}{\le}  \underset{(s_h,a_h)\sim (P''_w,\pi_w)}{\mathbb{E}}\left[ \sum_{h=1} ^H g^n_h (s_h,a_h,w) \right] = V^{\pi_w}_{\mathcal{M}^{(g^n,P'')}(w)},
\end{align*}
where $(i)$ follows from the definition in \cref{def:doublebar} and $(ii)$ follows by iteratively extracting the terms $g_h^n(s_h,a_h,w)$ from the value function gaps.
\end{proof}

The following lemma \citep{dann2017similation_lemma} provides an expression on the difference of two value functions under different MDPs.
\begin{lemma} (Simulation Lemma). \label{simulation lemma}Suppose $P'$ and $P''$ are transition kernels of two MDPs, and $r', r''$ are the corresponding reward functions. Given any policy $\pi$, we have :
\begin{equation*}
\begin{aligned}
V_{h,P',r'}^\pi& (s_h) - V_{h,P'',r''}^\pi (s_h)\\
& =\sum_{h'=h}^H \underset{\substack{s_{h'} \sim (P'', \pi) \\ a_{h'}\sim \pi}}{\mathbb{E}}\left[r'\left(s_{h'}, a_{h'}\right)-r''\left(s_{h'}, a_{h'}\right)+\left(P'_{ h'}-P''_{h'}\right) V_{h'+1, P', r'}^\pi\left(s_{h'}, a_{h'}\right) \mid s_h\right] \\
&=\sum_{h'=h}^H \underset{\substack{s_{h'} \sim (P', \pi) \\ a_{h'}\sim \pi}}{\mathbb{E}}\left[r'\left(s_{h'}, a_{h'}\right)-r''\left(s_{h'}, a_{h'}\right)+\left(P'_{h'}-P''_{h'}\right) V_{h'+1, P'', r''}^\pi\left(s_{h'}, a_{h'}\right) \mid s_h\right].
\end{aligned}
\end{equation*}
\end{lemma}

We next present a widely-used lemma for linear MDPs here, which is Lemma~G.2 in \citet{agarwal2020small_pmin} and Lemma~10 in \citet{uehara2021representation}.
\begin{lemma} (Elliptical Potential Lemma) \label{potential lemma}
Consider a sequence of $d\times d$ positive semidefinite matrices $X_1,...,X_N$ satisfying $\mathrm{tr}(X_n)\le1$ for any $n\in[N]$. Define $M_0 = \lambda_0$ and $M_n = M_{n-1}+X_n$. Then we have:
\begin{equation*}
    \sum_{n=1}^N \operatorname{tr}\left(X_n M_{n-1}^{-1}\right) \leq 2 \log \operatorname{det}\left(M_N\right)-2 \log \operatorname{det}\left(M_0\right) \leq 2 d \log \left(1+\frac{N}{d \lambda_0}\right).
\end{equation*}
\end{lemma}

\end{document}